\documentclass{article}

\usepackage{amsmath}
\usepackage{amssymb}
\usepackage{amsthm}

\newcommand{\Learner}{\mathsf{Lrn}}
 
 \newcommand{\To}{\mapsto}

\newcommand{\bits}{\{0,1\}}
 \newcommand{\Exp}{\operatorname*{\mathbb{E}}}
\newcommand{\Ex}{\Exp}
\newcommand{\ol}{\overline}
\newcommand{\ERR}{\mathsf{ERR}}
\newcommand{\set}[1]{\left\{ #1 \right\}}

\newcommand{\remove}[1]{}

\newcommand{\cH}{{\mathcal H}}

\newcommand{\cS}{{\mathcal S}}

\newcommand{\cX}{{\mathcal X}}
\newcommand{\cY}{{\mathcal Y}}

     \usepackage[final]{neurips_2022}

\usepackage[utf8]{inputenc} % allow utf-8 input
\usepackage[T1]{fontenc}    % use 8-bit T1 fonts
\usepackage[colorlinks,linkcolor=cyan,citecolor=blue]{hyperref}       % hyperlinks
\usepackage{url}            % simple URL typesetting
\usepackage{booktabs}       % professional-quality tables
\usepackage{amsfonts}       % blackboard math symbols
\usepackage{nicefrac}       % compact symbols for 1/2, etc.
\usepackage{microtype}      % microtypography
\usepackage{xcolor}         % colors

\usepackage{tcolorbox}

\usepackage{amsmath}
\usepackage{amsthm}

\newcommand{\ifcomments}{\iftrue}

\ifcomments
\newcommand{\Mnote}[1]{{\color{teal} [{\bf  Mohammad:}  #1]}}
\newcommand{\Inote}[1]{{\color{brown} [{\bf  Idan:}  #1]}}

\else
\newcommand{\Mnote}[1]{}
\newcommand{\Inote}[1]{}
\fi

\newcommand{\SPV}{\mathsf{SPV}}
\newcommand{\PSPV}{\mathsf{PSPV}}
\newcommand{\Maj}{\mathsf{Maj}}
\newcommand{\Learn}{\mathsf{Lrn}}
\newcommand{\Adv}{\mathsf{Adv}}
\newcommand{\OPT}{\mathsf{OPT}}
\newcommand{\Agn}{\mathsf{Agn}}
\newcommand{\dist}{\mathsf{d_H}}

\newcommand{\VC}{\mathtt{VC}}

\newcommand{\eps}{\varepsilon}

\newcommand{\E}{{\mathbb{E}}}
\newcommand{\N}{{\mathbb{N}}}

\newtheorem{theorem}{Theorem}[section]
\newtheorem*{theorem*}{Theorem}
\newtheorem{proposition}[theorem]{Proposition}
\newtheorem{lemma}[theorem]{Lemma}
\newtheorem{corollary}[theorem]{Corollary}

\newtheorem{observation}[theorem]{Observation}

\newtheorem{remark}{Remark}

\theoremstyle{definition}
\newtheorem{definition}[theorem]{Definition}

\title{On Optimal Learning Under Targeted Data Poisoning}

\author{%
Steve Hanneke \\
Purdue University, USA \\
\texttt{steve.hanneke@gmail.com}
\And
Amin Karbasi \\
Yale University, USA and Google Research \\
\texttt{amin.karbasi@yale.edu }
\And
Mohammad Mahmoody \\
University of Virginia, USA \\
\texttt{mohammad@virginia.edu}
\And
Idan Mehalel \\
Technion, Israel \\
\texttt{idanmehalel@gmail.com}
\And
Shay Moran \\
Technion, Israel and Google Research \\
\texttt{shaymoran1@gmail.com}
}

\begin{document}

\maketitle

\begin{abstract}%
Consider the task of learning a hypothesis class $\cH$ in the presence of 
    an adversary that can replace up to an $\eta$ fraction of the examples in the training set with arbitrary adversarial examples. 
    The adversary aims to fail the learner on a particular target test point $x$ which is \emph{known} to the adversary but not to the learner. In this work we aim to characterize the smallest achievable error $\eps=\eps(\eta)$ by the learner in the presence of such an adversary in both realizable and agnostic settings. %and b) explore whether such error rates can be achieved by a proper learner. 
We fully achieve this in the realizable setting, 
    proving that $\eps=\Theta(\VC(\cH)\cdot \eta)$,
    where $\VC(\cH)$ is the VC dimension of $\cH$. 
    Remarkably, we show that the upper bound can be attained by a deterministic learner.
    % Thus, in the worst-case, the adversary can only force an error proportional 
    % to the VC dimension and the corruption budget $\eta$. 
%
In the agnostic setting we reveal a more elaborate landscape:
    we devise a deterministic learner with a multiplicative regret guarantee of
    $\eps \leq  C\cdot\mathtt{OPT} + O(\VC(\cH)\cdot \eta)$,
    where $C > 1$ is a universal numerical constant.
    We complement this by showing that for any deterministic learner there is an attack which worsens its error to at least $2\cdot \mathtt{OPT}$. 
    This implies that a multiplicative deterioration in the regret is unavoidable in this case.
    % In contrast with the realizable case, we show that no deterministic algorithm can enjoy \emph{additive} regret guarantees:
    % we show that for any deterministic algorithm there is an attack 
    % which deteriorates its error at least twice the error rate of the best in class. 
    % On the positive side, we devise an deterministic learning rule
    % whose loss is at most a constant factor times the error rate of the best in class plus a small additive term
    % (thus complementing the impossibility result).
    % The situation for the agnostic case is more involved. We prove that no \textit{deterministic} algorithm can achieve an arbitrary small error with respect to the best  hypothesis in class $\cH$. In contrast, we develop a deterministic algorithm that can achieve twice the error rate of the best in class, with an arbitrary small additive error, thus proving that semi-agnostic learning is possible.  
%
Finally, the algorithms we develop for achieving the optimal rates are inherently improper. Nevertheless, we show that for a variety of natural concept classes,
    such as linear classifiers, it is possible to retain the dependence
    $\eps=\Theta_{\mathcal{H}}(\eta)$ by a proper algorithm in the realizable setting. Here $\Theta_{\cH}$ conceals a polynomial dependence on $\VC(\cH)$. 

\end{abstract}
%\tableofcontents

\section{Introduction}

A basic goal in machine learning is to develop a predicting model from labeled examples (i.e., training data) that can reliably generalize to unseen examples (i.e., test data). In its simplest form, namely, binary classification, a learner $\Learn$ is given a training set $S = \{(x_1,y_1), \dots, (x_n,y_n)\}$, usually assumed to be i.i.d. samples drawn from an unknown distribution $D$ of labeled examples where $x_i$'s are the domain instances (or data points) and $y_i\in\{0,1\}$ are the labels.  The aim is to produce  a   mapping $h = \Learn(\cS)$ that predicts the labels of fresh examples $(x,y) \sim D$ as accurately as possible, i.e., to minimize the population loss $L_D(h) = \Pr_{(x,y) \sim D}[h(x)\neq y]$. This classical setting has been extensively studied in the last half a century.
This accumulated work resulted in fundamental mathematical characterizations regarding the nature of learnability  when the training samples  are truly i.i.d without any tampering by an adversary \citep*{shalev2014understanding}. The goal of this paper is to offer a similar characterization in the presence of an adversary who can tamper with a subset of the training data. 
%It is well-known  that $o_n()$-error classification (using sufficiently large sample $S$) in this setting is possible, if and only if the function $f$ that computes the true labels $(x,f(x)) \sim D$ comes from a ``hypothesis'' set $\cH$ of small VC dimension \cite{xyz}. 
%\textcolor{red}{Shay: I would omit the last sentence as it may antagonize some readers:
%the VC dimension characterization applies to the PAC model which has been criticized for focusing on the worst-case. It is true that our results apply in the PAC model, but I suggest a different narrative which can be summarized as follow: despite an extensive line of work on learning in the presence of adversarially poisoned data, the basic problem of characterizing the optimal achievable error remains open, even in the classical and basic PAC learning model\ldots We solve this problem\ldots} \Mnote{sure; removed.}

With the emergence of sensitive machine learning applications, it is critical to ensure the trustworthy of such predictive models in the non-ideal scenarios. In this paper, we consider \emph{robust} learnability when the training examples can be altered by an adversary whose goal is to make sure that a target test point will be predicted incorrectly.
For instance, a language model trained on conversations in shopping forums can be attacked by marketing campaigns, who may want a specific product to be associated with a positive experience, instead of a bad one.
Another example is an adversary who aims to fool a self-driving car to speed up once it observes a stop sign. If such an adversary can somehow influence the training sets used for training the decision rules, she has all the reasons to strategically change them with the specific goal of misleading the self-driving car.
As another example, consider a loan applicant who wants to make sure that his loan will be granted. If he can somehow change the training set used by the bank, he might be able to make his application approved. 
%\textcolor{red}{Perhaps in one example use ``he'' rather than ``she''? some readers are sensitive and might not like the fact that both the adversary and deciever are women\ldots} \Mnote{I also have seen people using "it" for the adversary, which is fine i think.}
Note that the training set is the lens through which a learning algorithm obtains information about the underlying learning process. Therefore, once we allow the training examples to be tampered with by an adversary, even slightly, unexpected outcomes may take place. To quantify the robustness of learning algorithms, in this paper, we show how much the outcome of a learning algorithm for a particular target test can be trusted once the training set is being altered. 

\textbf{(PAC) learning under instance-targeted poisoning.} More formally, we consider an adversary  $\Adv$ that is allowed to \emph{replace} an $\eta$-fraction of the training sample $S$, resulting to a tampered training sample $S'$ given to the learning algorithm $\Learn$. Note that even though the training sample $S$ is drawn i.i.d. from a distribution $D$, the tampered training sample $S'$ does not enjoy this property anymore. 
%\Inote{I guess that it should be stressed here that our results only applicable to replacements attacks?}
Such attackers are also called  \emph{poisoning} adversaries \citep*{barreno2006can}, and variants of them are previously studied under the name of \emph{malicious} noise \citep*{valiant1985learning,kearns1993learning} or \emph{nasty} noise \citep*{bshouty2002pac}. More specifically, we study poisoning settings in which the adversarial perturbation of the original sample $S$ \emph{can also depend on the final test instance} $x$. Due to the adversary's knowledge of the target test point $x$, such poisoning attacks are sometimes referred to as \emph{instance-targeted} poisoning attacks \citep*{barreno2006can}. Even without any manipulation to the training set, it is too much to ask the learning algorithm to predict correctly all the time while given only a finite number of examples to learn from. In the same vein, we can only hope to design a robust learning algorithm that is correct with high probability over the selection of $(x,y) \sim D$, especially if the adversary knows the test instance $(x,y)$ before manipulating the training set $S$ to $S'$.
\citet*{gao2021learning}, building on ideas from \citep*{levine2020deep}, proved that PAC learnability  under instance-targeted poisoning attacks is achievable only when $\eta=o(1)$. In other words, when the adversary can only change a \emph{sublinear} $o(n)$ number of  $n$ examples, then the optimal learner can achieve error $o(1)$ that goes to zero when the number of examples $n$ goes to infinity.

\subsection{Our Results}

The prior work leaves several key questions open on the exact parameters of learnability under instance-targeted poisoning. Most importantly, the work of \citet*{gao2021learning} does not quantify the error rate when the adversary's budget is $\eta=\Omega(1)$ (e.g., if the adversary can corrupt $n/100$ of the examples). Secondly, \citet*{gao2021learning} only assume the realizable setting as it is crucial for their results that all the ``sub-models'' trained using the bagging technique will have error that goes to \emph{zero}. Hence, the  question of finding optimal learning rates is left open for both realizable and agnostic settings. Finally, as the developed robust algorithms are all based on ``bagging'' they are inherently improper learning technique.  

In this work, we make progress on all the directions above and achieve optimal error rates (up to constant factors) for general $\eta$, both for the realizable and agnostic settings. We further study the proper nature of the obtained algorithms and give the first proper learning methods that are robust against instance-targeted poisoning attacks for natural hypothesis classes such as linear classifiers.
%proper learning, and (3) agnostic learning.  \textcolor{red}{Which questions? Are there only two of them? } \Mnote{moved them to this subsection and called them 'directions'.}
%
%\paragraph{Characterizing the optimal error rate.}
More precisely, we give a characterization of the optimal error rate of learning  under instance-targeted poisoning attacks with budget $\eta \cdot n$ as follows.

\textbf{Realizable setting.}
We show that the optimal error
% \footnote{Optimal error $\eps$ here means that there is a learner that guarantees error $\leq \eps$, and for every learner there is an instance-targeted adversary that increases  the error to $\geq \eps$.} 
is $\Theta(\eta \cdot d)$ where $d$ is the VC dimension of the hypothesis set $\cH$. To prove this, we first present an upper bound, showing that a (deterministic) learner can guarantee the error to be at most $O(\eta d)$ under any instance-targeted poisoning attacks of budget $\eta  n$. We then also show a matching  lower bound (up to a constant factor) as follows. For any \emph{nontrivial}\footnote{A non-trivial class $\cH$ is one for which there are $x_1,x_2\in \cX$ and $h_1,h_2 \in \cH$ so that $h_1(x_1)=h_2(x_1)$ and $h_1(x_2) \neq h_2(x_2)$. In particular, any class containing at least $3$ hypotheses is non-trivial.} hypothesis class of VC dimension $d$, we show how to design a distribution $D$ over the examples such that no matter how the learning proceeds, there always exists an adversary of budget $\eta  n$ that can increase the error (under the instance-targeted attack) to $\Omega(\eta d)$. Our lower bound above holds even if the learning algorithm uses \emph{private randomness} that is not known to the adversary\footnote{This model is referred to as the ``weak'' learning model (under instance-targeted poisoning attacks) in the work of \citet*{gao2021learning}.}. Our positive result, however, is deterministic, and so can be seen as satisfying the \emph{stronger} guarantee, in which the adversary's perturbations to the training set is allowed to depend on learner's randomness.

\textbf{Agnostic setting.} We also extend our result above to the agnostic setting in which all hypotheses  $h \in \cH$ have population loss bounded away from zero (even before the attack). In this setting, we devise a deterministic algorithm whose expected error on the test point is $O(\OPT + \eta \cdot d)$, where $\OPT$ is the population loss of the best hypothesis $h \in \cH$. 

A natural question that arises is whether one can achieve an additive regret guarantee of $\OPT + O(\eta \cdot d)?$ (Note that agnostic learning is usually defined with respect to additive regret). 
We show that this is in fact \emph{not possible}, at least for deterministic learners, by presenting a negative result. In particular we show that for any deterministic learner $\Learn$, there is an extremely simple hypothesis class (just consisting of two functions) and an input distribution such that the learner is forced to have adversarial error $\geq 2 \OPT$.
    This negative result uses tools from the computational concentration of products \citep*{talagrand1995concentration} 
    and a continuity intermediate-value argument.

\textbf{Proper learning.} 
The deterministic algorithm witnessing the above upper bound is inherently improper which might be a disadvantage 
in terms of interpretability or test-time computational complexity. In contrast, in (the non-adversarial) PAC setting
proper algorithms are known to achieve near optimal learning rates (up to log factors). 
We therefore explore the cost of proper learning under instance-targeted poisoning attacks. 
We show that in many natural classes, such as half spaces, 
it is indeed possible to obtain proper learning rules that are robust to instance-targeted poisoning attacks,
with guarantees which are only polynomially worse than optimal.
For example, for the class of half-spaces in $\mathbb{R}^d$ we derive a deterministic proper learning rule
whose error rate is at most $O(d^3 \eta)$. 
At a technical level, we achieve this result by relying on the \emph{projection number} of the class~\citep*{bousquet2020proper,kane2019communication, braverman2019convex}. 
% In particular, when the projection number of the hypothesis set is small enough, one can project any majority-based hypothesis with high confidence back to the hypothesis set without introducing much error. 
%\textcolor{red}{How about the result for general hypothesis classes? (with the square root)?}

%\paragraph{Semi-agnostic learning.} Finally, we show that even though  the agnostic PAC learning under targeted poisoning  is impossible in general, it is nevertheless possible to achieve a less demanding task. In particular, we  develop a learning algorithm whose error is within a constant multiplicative factor of that of the best hypothesis in class. \textcolor{red}{We also show that a multiplicative constant is necessary, right?}

\subsection{Relation to Certification and Stability}
\paragraph{Certification.} 
Robustness to instance-targeted poisoning boils down to the following type of stability: 
    on most of the test instances $x$, the prediction of the learner $y=y(x)$ remains the same even if at most $\eta$ fraction of the examples in the training-set $S$ are replaced.
    It is natural to require the learning rule to \emph{certify} this stability. 
    That is, a certifying learning rule provides a bound $k=k(x)$ along with the prediction label $y=y(x)$, 
    where the meaning of $k$ is that the prediction $y=y(x)$ remains the same
    even if at most $k$ examples in the input sample are replaced.
    Note that it is always possible to provide the trivial guarantee of $k=0$,
    and therefore the goal is to design robust learners that provide non-trivial certificates.
%\textcolor{red}{Shay: I don't understand this sentence (the first part is clear, but I don't understand the second part)}. \Mnote{modified it.} 
    Our algorithm naturally achieves that: for $\approx 1-\eps$ of the test instances $x$ it provides a guarantee of $k\approx \eta n$.

% Robustness to targeted poisoning provides the guarantee that most of the test instances have a stable prediction when the database is changed by a bounded amount in Hamming distance. A stronger guarantee is to \emph{certify} this robustness on a point by point basis. In other words, a certifying predictor provides a bound $k$ along with the prediction label $y$, such that any sample $S'$ with Hamming distance at most $k$ would lead to the same prediction label $y$. The certification bound $k$ does not have to be either $0$ or  a constant $k$, and it could, in general, depend on the test instance $x$. So, in a sense, a certifying learner is useful if it provides large certification bounds $k$ for more test instances (Note that it is always possible to provide $k=0$, since no changes in the training set will not lead to a change in the prediction).
% %\textcolor{red}{Shay: I don't understand this sentence (the first part is clear, but I don't understand the second part)}. \Mnote{modified it.} 
% Our results extend to certifying predictors that provide $k\approx \eta n$ robustness for $\approx 1-\eps$ of the test instances.

\textbf{Connection to stability.} We also present a new perspective on instance-targeted poisoning attacks by showing how they can be seen as natural forms of algorithmic stability \citep*{bousquet2002stability,rakhlin2005stability}. In particular, we show that one can study the adversarial robustness (around the \emph{true} label) to instance-targeted poisoning by decoupling the (pure) stability aspect (which does not depend on the true labels) from the (non-adversarial) risk. We refer to the former as the \emph{prediction stability}.  Roughly speaking,  prediction stability   requires that the model's prediction on $x$ does not change even if the adversary changes the training set withing its budget $\eta n$. Note that here we do not care whether the model's output on $x$ is the correct label or not, and hence is a pure measure of stability of the predictions. 

It might be helpful to compare prediction stability with the algorithmic stability of \citep*{bousquet2002stability,rakhlin2005stability}. 
The later requires that for a typical sample $S$ of size $n$, and for every \emph{fixed} $i \in [n]$, 
the prediction of the model trained on $S$ and tested on a random test-point $x$ is likely not changed 
if one substitutes the $i$-th example in $S$ with a \emph{fresh} random example. 
Prediction stability strengthens this condition in two ways: (1) the choice of what coordinate in $S$ to change can adversarially depend on the test instance $x$, (2) the adversary is allowed to change \emph{more}  than one examples (i.e., up to $\eta \cdot  n$).

%\paragraph{Potential applications.} (eg to privacy).

\subsection{Related Work}  

Poisoning attacks are studied in theoretical learning  under various noise models~\citep*{valiant1985learning,kearns1993learning,Sloan::Noise:four-types,bshouty2002pac}. However, these works focus on the \emph{non-targeted} setting in which the adversary does \emph{not} know the target instance.  

The \emph{computational} aspects of efficient learning 
under (non-targeted) poisoning have been studied in 
various works, including that of 
\citet*{kalai:08,klivans:09,awasthi:14}, with this last work 
obtaining nearly optimal (up to constants) 
learning guarantees among polynomial-time algorithms
for learning homogeneous linear separators 
with malicious noise 
under distribution restrictions.
That result was subsequently extended to the 
\emph{nasty noise} model by 
\citet*{diakonikolas:18}, 
via techniques that also enable them 
to study other geometric concept classes.
In the \emph{unsupervised} setting, ~\citet*{diakonikolas2016robust,lai2016agnostic}
studied the computational aspect of learning under   poisoning.
%(also see the survey~\citep{diakonikolas2019recent}).
%(also see the survey~\citep{diakonikolas2019recent}), 
%the \emph{computational} aspects of efficient learning under (non-targeted) poisoning is studied.  
In contrast, our work focuses on (supervised) instance-targeted poisoning, and we study the learning rates \emph{information theoretically} regardless of learner's computing power. The work of~\citet*{steinhardt2017certified} further studied the {certification} of the overall (non-targeted) error. More recently, such (non-targeted) poisoning attacks are combined with \emph{test-time} attacks and are studied under the name of \emph{backdoor}  attacks~\citep*{gu2017badnets,ji2017backdoor}.

Besides instance-targeted attacks (which are the focus of this paper), other notions of targeted attacks were studied in the literature: for example, in \emph{model-targeted} attacks, the adversary's goal is to make the learner predict according to a specific model. Recent works on this model include \citep*{farhadkhani2022equivalence, suya2021model}. Some other works study \emph{label-targeted} attacks, in which the adversary's goal is to flip the decision on the test instance to a specific label (e.g., see \emph{targeted misclassification} attacks in  \citep*{chakraborty2018adversarial}).
%\Inote{I did not find a reference for the "label-targeted" setting that was mentioned by one of the reviewers... Are we sure that it was studied?}
The work of \citep*{jagielski2021subpopulation} studies a generalization of instance-targeted attacks, called \emph{subpopulation} attacks, in which the adversary knows the subset of the inputs, from which the test instance will be drawn.

Most relevant to our setting are the recent works of \citet*{gao2021learning,blum2021robust} where the general problem of learning (and more quantitative variant of learning error rate) under \emph{instance-targeted} poisoning was formally defined and studied. In particular, \citet*{blum2021robust} studied learnability under instance-targeted poisoning where the adversary can add an \emph{unbounded} number of so-called clean-label 
%\shay{What does clean-label means?}
examples to the training set. A clean-label example $(x,y)$ has the property that $y$ is the \emph{correct} label of $x$, while $x$ could be an arbitrary instance that is \emph{not} sampled from the same distribution that generates other instances in the training set.
%\Mnote{added explanation for clean label}
\citet*{gao2021learning} also showed that when the adversary's corruption is only an $o(1)$ fraction of the training set, PAC learning is possible (if it is possible without the attack). In a concurrent work, \citet*{balcan2022robustly} study the problem of certifying the \emph{correct} prediction  even under instance-targeted data poisoning.
%\shay{I don't understand the last sentence.}
Our methods, however, can be used to obtain certification of the stability of the model  around their prediction (even though the prediction might \emph{not} be true always), while controlling the overall error to be provably small (again under the instance-targeted attack).
%\Mnote{I added clarifications for correct-prediction certificatgion.}

%In a closely related line of work, theoretical ideas (such as randomized smoothing) were used to \emph{empirically} study  (even certification of) robustness against instance-targeted poisoning. 
%\shay{I don't understand this sentence.} \Mnote{I removed the sentence before, and clarified the ones after.}
\citet*{rosenfeld2020certified} empirically demonstrated that randomized smoothing \citep*{cohen2019certified} can provide robustness against label-flipping attacks, in which the adversary is  limited to merely flipping the label of a subset of the training set. They also showed that randomized smoothing can be used to handle \emph{replacing} attacks (the model also studied in this paper), in which the adversary substitutes a part of the training set with a new set of same size. Subsequently,~\citet*{levine2020deep} used deterministic methods that further allowed   attacks that can add examples to or remove them from the training set. \citet*{chen2020framework,weber2020rab,jia2020intrinsic} further developed the technique of {randomized} bagging/sub-sampling for the goal of resisting instance-targeted poisoning attacks.

Finally, we comment that other theoretical works  have also studied instance-targeted poisoning attacks~\citep*{mahloujifar2017blockwise,
%mahloujifar2018learning,mahloujifar2019universal,mahloujifar2019can,mahloujifar2019curse,diochnos2019lower,
etesami2020computational}.  These works show how to \emph{amplify} error for specific test instances, say from $0.01$ error to $0.5$, through instance-targeted poisoning. In particular, these works do not talk about the \emph{fraction} of the test population that is vulnerable to targeted poisoning.   The work of~\citet*{shafahi2018poison} studied the power of such  attacks empirically.

\remove{
\subsection{ideas for more material}
Some thoughts and suggestions regarding intro + presentation of main results.
\begin{enumerate}
    \item Begin with a catchy example which demonstrates the potential relevance of the model.
    (Perhaps something related to self-driving automobiles?)
    \item Already in the intro explain the equivalence with local stability, and discuss other potential applications of local stability. (Eg in active learning, privacy, fairness, etc.) \Mnote{Do you mean defining such stability for active learning? and stating the potential that local stability can be used to derive other robustness measures such as privacy and fairness?}
    \textcolor{red}{Shay: I mena to define this notion of stability in general and to note that besides its relevance to targeted attacks, it might also have other applications. Then, to give examples such as active learning, privacy, etcetera. What I have in mind is a couple of informal paragraphs, aiming to inspire the reader, without getting into formal details.}
    Uri Stemmer mentioned to me a relevant definition given by Adam Smith which is called something like ``distance from instability''. It is worth looking into it and commenting about it in this context.
    (Perhaps Amin/Mohammad are familiar with this concept?) \Mnote{I am not familiar with this, beyond the fact that Adam (and others perhaps) define privacy as a form of stability.} \textcolor{red}{Shay: Okay, then perhaps it is worth looking into it. I'll see Uri today at Google and ask for explicit references.}
    \item When we state our main results, emphasize both the data poisoning and local stability aspects.
    It is also important to note that our algorithms (both improper and proper) outputs a \emph{certifying} function with a lower bound on the local-stability. \Mnote{So do we want to define certification too? There is a tricky thing here that we discussed last time: when we don't care about the efficiency of the learner and hypothesis, then certification comes for free, because one can enumerate all possible neighboring sets. But of course we do get poly-time certification, which is not for free. Steve also previously observed that when we don't care about efficiency, we can also avoid using randomness and handle $\eta \cdot n$ addition/removal. This also resolves a question that previous work has not done.}
    \textcolor{red}{Shay: I think we should define this notion of stability and discuss the difference between certified and uncertified stability, and stress that our algorithms are of the former kind. (Note that in general it is not even clear whether one can even compute the stability, even when given unbounded computational resources, because the search space is infinite.)} \Mnote{Right. what I meant is that in the information theoretic notion, in which the learner is a function, regardless of how it is computed, then the certification comes for free.}
    \item Highlight the open problem regarding proper learning, discuss the challenges, and briefly describe how we are able to circumvent them in the context of linear classifiers with margin. (By providing a strong majority-vote and projecting it back to the class.)
    \item Contrast the bounds we obtain with the known bound in the setting of non-targeted data-poisoning. \Mnote{Good: here there is an extra factor $d$ (beyond the constants).}
\end{enumerate}
}
\section{Preliminaries} \label{sec:prelim}
%This section is organized as follows. First we define some notation, as well as the notion of \emph{Adversarial risk}, which is the main concept discussed in this paper. Then, we present the notion of \emph{Prediction-stability} which has a connection to adversarial risk, and might be interesting on its own right. Finally, we present two learning rules used in this paper.
%\subsection{Definitions}

\textbf{Notation and basic learning theory definitions.}
We consider the setting of binary classification. 
    Let $\cX$ denote the input domain   and $\cY=\{0,1\}$ denote the label-set. 
    A pair $(x,y)\in \cX \times \cY$ is called an \emph{example}. 
    A sequence $S=(x_1,y_1), \dots, (x_n,y_n) \in \left(\cX \times \cY\right)^n$ of $n$ examples is a \emph{sample} of size $n$. 
    The $i$'th example in $S$ is denoted by $S_i$.

A function $h\colon \cX \to \cY$ is called an hypothesis or a concept. 
    A set of hypotheses $\cH \subset \cY ^{\cX}$ is called an \emph{hypothesis class}, or a \emph{concept class}. We denote the VC-dimension of a concept class $\cH$ by $d=d(\cH)$.

For a set $Z$, let $Z^*=\cup_n Z^n$ denote the set of all finite sequences with elements from $Z$.
    A \emph{learning rule} or \emph{learning algorithm} or \emph{learner} $\Learn \colon (\cX \times \cY)^* \rightarrow \cX ^{\cY}$ 
    is a deterministic\footnote{In Appendix~\ref{apndx:neg_proof}, we extend the definition in a way that captures also a family of randomized learners.} mapping which takes an input sample $S \in (\cX \times \cY)^*$ 
    and maps it to a hypothesis $\Learn(S)=h\in \cX ^{\cY}$. 
    %We denote the function that $\Learn$ outputs     when fed with the sample $S$ by $\Learn(S)$. 
    If it is guaranteed that $\Learn(S) \in \cH$ 
    for all input samples $S$ then $\Learn$ is said to be \emph{proper}; otherwise, it is \emph{improper}. 

Let $D$ be a distribution over examples, and let $h$  be an hypothesis.
The \emph{population loss} of $h$ with respect to $D$ 
is defined by $L_D(h)= \Pr_{(x,y) \sim D}[h(x)\neq y]= \E_{(x,y)\sim D}[1[h(x)\neq y]]$.
A distribution $D$ is said to be \emph{realizable} by $\cH$ if $\inf_{h\in \cH}L_D(h)=0$. 
Similarly, for a sample $S$, let $L_S(h)=\frac{1}{\lvert S\rvert}\sum_{i=1}^n1[h(x_i)\neq y_i]$
denote the \emph{empirical error} of $h$ with respect to $S$, and call a sample realizable by a class $\cH$
if there exists $h\in \cH$ such that $L_S(h)=0$.
The expected loss (also called risk) of a learning algorithm $\Learn$ w.r.t a distribution $D$ and sample size $n$ is defined by
\[\eps_n(\Learn \vert D):=\Pr_{S \sim D^n, (x,y) \sim D} \left[ \Learn(S)(x) \neq y \right].\]
The function $n\mapsto \eps_n(\Learn \vert D)$ is called the \emph{learning curve}, or \emph{learning rate} of $\Learn$ w.r.t $D$.
 
For a real number $r$, let $\lfloor r \rceil$ denote the nearest integer to $r$. In case of ties, when $r=k+1/2$ for some $k\in \mathbb{Z}$, then define $\lfloor r \rceil = k+1$. For any finite multiset $\cH' \subset \cH$, denote by $\Maj(\cH')$ the function defined for all $x\in \cX$ by $\Maj(\cH')(x) = \left \lfloor \frac {1}{\lvert \cH' \rvert} \sum_{h' \in \cH'} h'(x) \right \rceil$.

\textbf{Adversarial risk and prediction stability.}
Before we introduce the definition of Adversarial risk, we define \emph{Hamming distance} between samples, which is a natural way to quantify distance between samples of equal size.

\begin{definition}[Hamming distance between samples]
Fix $n\in \N$ and let $S,S'\in \left(\cX \times \cY \right)^n$. We define the \emph{Hamming distance} between $S$ and $S'$ by $\dist(S,S') = \sum _{i=1}^n 1[S_i \neq S'_i]$. 

Note that the Hamming distance is defined only for samples of equal sizes. If $\dist(S,S') \leq \eta \cdot n$, we say that $S,S'$ are \emph{$\eta$-close}. 
For any sample $S$, let $B_\eta (S) := \bigl\{S' : \dist(S,S')\leq \eta\cdot n\bigr\}$.
\end{definition}

\begin{definition}[$\eta$-adversarial risk] \label{def:risk}
Let $\eta\in(0,1)$ be the adversary's budget, let $\Learn$ be a learning rule, 
and let $D$ be a distribution over examples. The \emph{$\eta$-adversarial risk} of $\Learn$ w.r.t $D$ and sample size $n$ is defined by
\[
\eps^{\Adv}_n(\Learn \vert D,\eta):=
\Pr_{S \sim D^n, (x,y) \sim D} \left[ \exists S' \in B_{\eta}(S): \Learn(S')(x) \neq y \right].
\]
% Note that in the above definition, the set $S'$ is chosen in a way enforces the learner to make a mistake on a targeted data point $x$. In other words, given a training set $S$ and a target point $x$, we search for a tampered training set $S'$, withing the budget of the adversary, with the hope of forcing the learning algorithm to make a mistake. 
%\textcolor{red}{TODO: is there a standard notation for adversarial loss rather than the above $\eps_{D,\eta,\Learn}$?}
%\Mnote{not really. there is only 2 papers formalizing this (Steve's paper and our UAI), and they use different notations.}
\end{definition}
Thus, robust learning with respect to instance-targeted poisoning with budget $\eta$ boils down to minimizing the adversarial risk.
Indeed, given an input sample $S$ and a test example $(x,y)$, an adversary with budget $\gamma$ can force a mistake on $x$ 
if and only if $\Learn(S')(x)\neq y$ for some $S'\in B_{\eta}(S)$.

\textbf{Randomness.} In Definition~\ref{def:risk} above we define adversarial risk for the setting in which both the  learner $\Learn$ and the model $h=\Learn(S)$ are \emph{deterministic}. When either $\Learn$ or $h$ is allowed to use randomness, then the notion of adversarial risk as defined in Definition~\ref{def:risk} can be extended in several ways, depending on whether the adversary can see the randomness of the learner or not. Some of these variations are discussed in the work of~\citet*{gao2021learning}. We remark however that our results in the realizable setting apply to all variations. This is simply because our upper bounds are achieved by deterministic learners, whereas our lower bound uses the weakest type of an adversary (which does not depend on the randomness of the learner). In contrast, our lower bound in the agnostic setting applies only to deterministic learners.

\textbf{Explicit bounds.}
We do not try to optimize the constants hidden in the $O(\cdot), \Omega(\cdot)$ notation in the derived bounds. 
The reason is because on the one hand, this way the proofs are simpler and more accessible, 
and on the other hand, we do not know how to get tight (or nearly tight) lower and upper bounds on the constants. Obtaining tight bounds is a natural direction for future research; we elaborate on this in Section~\ref{sec:conc}. 
Nevertheless, the complete proofs (which are given in the appendix) include explicit numerical bounds on the constants.

\textbf{Decoupling adversarial risk into stability and  risk.} It is convenient and illustrative to decouple robust learnability to two properties: small expected loss and prediction stability. 
%As we explained briefly earlier, prediction stability is a new notion that is quite different from the usual stability studied in learning theory. 
The latter means that the prediction of the learning algorithm on a random test point is stable under replacing a bounded amount of examples from the training set: %We stress that the property of being prediction-stable is orthogonal to the property of having small loss. 
%The way we ensure prediction stability is by taking the majority vote over sufficiently many classifiers trained on disjoint samples. 

\begin{comment}
\textcolor{red}{TODO: elaborate why (common in other contexts such as private ML, definition of stability might be interest in its own right (eg in the context of generalization), provides a more general language in which one can study prediction-stable algorithms.) Another point is that our analysis applies separates arguments to get prediction-stability (majority vote over sufficiently many classifiers trained on disjoint samples) and low error (Markov argument). }
\end{comment}

\begin{definition}[Prediction stability]
Let $n\in\mathbb{N}$, $\sigma, \eta \in (0,1)$. Let $\Learn$ be a learning rule and $D$ be a distribution over examples.
We say that the learning rule $\Learn$ is \emph{$(n,\sigma, \eta)$-prediction stable with respect to $D$} if the following holds
\[
\lambda_{n}(\Learn \vert D, \eta) :=  \Pr_{S \sim D^n, x \sim D_x} \left[ \exists S' \in B_{\eta}(S): \Learn(S')(x) \neq \Learn(S)(x) \right] \leq \sigma.
\]
%\Mnote{The probability above is very similar to adversarial risk. so, how about we give it a similar notation like $\gamma_n(\Learn \vert D,\eta)$? that helps the reader to compare them more easily, and let's us simplify how we write things like the observation below simply into:}
where $D_x$ is the marginal distribution induced by $D$ on the domain $\cX$.

% We say that $\Learn$ is \emph{$(n_0,\sigma, \eta)$-prediction stable} if it is $(n_0,\sigma, \eta)$-prediction stable with respect to
% every distribution.
\end{definition}
Of course, prediction stability alone does not guarantee robust learning. 
Indeed, useless learning rule that always outputs the all $0$'s classifier has maximal stability. 
At the very least, the learning rule should learn the class in the classical sense (in the absence of an adversary). 
The following observation asserts that prediction-stable learning rules with small loss are robust learners: 

\begin{observation}[Prediction stability + small error $=$ robust learning] \label{obs:stab_learn_robust}
Let $\Learn$ be a learner and $D$ a distribution over examples. Then,
\[\max \set {\eps_{n}(\Learn \vert D), \lambda_{n}(\Learn \vert D, \eta) }\leq \eps_{n}(\Learn \vert D, \eta)\leq \lambda_{n}(\Learn \vert D, \eta) + \eps_{n}(\Learn \vert D).\]
\end{observation}

 In other words, if $\Learn$ is $(n,\sigma, \eta)$-prediction stable with respect to $D$ whose expected population loss is $\eps_{n}(\Learn \vert D)\leq \eps$. Then $\Learn$ learns $D$ with an adversarial expected loss $\sigma + \epsilon$.
%\[\eps_{n}(\Learn \vert D, \eta)\leq \sigma + \epsilon.\]
Conversely, if $\eps_{n}(\Learn \vert D, \eta)\leq \eps$ then $\Learn$ is $(n,\eps, \eta)$-prediction stable with respect to $D$ and its expected population loss is also $\eps_{n}(\Learn \vert D)\leq \eps$. We leave the (simple) proof of Observation~\ref{obs:stab_learn_robust} to the reader.

\section{Realizable Setting}

Theorems \ref{theo:pos} and \ref{theo:neg} below characterize the optimal adversarial risk in the realizable setting.

\begin{theorem}[Realizable case -- positive result] \label{theo:pos}
There exists a constant $c_1 > 0$ so that the following holds. Let $\cH$ be a hypothesis class with VC dimension $d$ and let $\eta\in (0,1)$.
Then there exists a learner $\Learn$ having $\eta$-adversarial risk
\[
\eps^{\Adv}_{n}(\Learn | D, \eta) \leq c_1 \eta d
\]
for any distribution $D$ realizable by $\cH$ and for any sample size $n \geq 1/\eta$.
\end{theorem}
We prove Theorem~\ref{theo:pos} in Appendix~\ref{sec:pos_proof}.
% \textcolor{blue}{Amin: We need to interpret this result in terms of targeted poisoning. How is this related to our UAI result regarding PAC learnability when sub-linear of the training data points can be tampered with?}
% \textcolor{red}{Shay: Agreed. Perhaps we can elaborate after the equation: ``In more detail, there exists a learning algorithm satisfying the following: if the stability parameter (or the adversary's budget) $\eta$ statisfies\ldots and the training-set size $n$ satisfies\ldots then\ldots''}

Note that the requirement that the sample size is $n \geq 1/\eta$     
    is necessary since otherwise $\eta\cdot n < 1$, which means that the adversary 
    cannot modify the input sample, and so this case reduces to classical learning without an adversary.

% learning under poisoning attacks is reduced 
% to standard learning: it holds that $n < 1/\eta$ if and only if 

Theorem~\ref{theo:pos} is proven using the \textsc{Stable Partition and Vote} (or $\SPV$, for short) meta-algorithm, described in Figure~\ref{fig:SPV}. The meta-algorithm is based on the idea of partitioning and then voting used in \citep*{gao2021learning}, but with a more refined and precise analysis.  The partition and vote technique works as follows. First, partition the input sample to subsamples of a carefully chosen size. Then, train a given learner (which is called the \emph{input learner} of $\SPV$) on each subsample, and finally let the trained learners vote to determine the output label. The size of each subsample trades-off, in a way, expected loss and prediction-stability: if it is too small, the given learner will perform poorly on each subsample. On the other hand, if it is relatively large then the number of learners that participate in the majority vote is small
and the adversary can poison a large fraction of these learners and flip the overall majority vote. We elaborate on this when proving Theorem~\ref{theo:pos}. Notice that the time complexity of $\SPV$ is proportional to the time complexity of the  learner $\Learn$.

% We are interested in both reasonable learnability and prediction-stability to attain near-optimal adversarial risk, and reflect this interest in the size we choose for each slice, which is $\Theta(1/\eta)$. 

\begin{figure}
    \centering
    \begin{tcolorbox}
    \begin{center}
        $\SPV$: \textsc{Stable Partition and Vote}
    \end{center}
    \textbf{Input:} Stability parameter $\eta \in (0,1)$, a learning algorithm $\Learn$ and an input sample $S \sim D^n$ where $n \geq 1/\eta$.
    \\
    \textbf{Output:} A classifier $h\colon \cX \to \cY$.
    \begin{enumerate}
        \item Partition $S$ into $\lceil 7 \eta n \rceil$ consecutive subsamples such that all first $t=\lfloor 7 \eta n \rfloor$ subsamples are of size at least $\frac{1}{7 \eta}$. Denote the $i$'th subsample by~$S^{(i)}$.
        \item For all $i\in [t]$, run the learning algorithm $\Learn$ on $S^{(i)}$ to obtain a hypothesis $h_i = \Learn(S^{(i)})$.
        \item Return the hypothesis $h$ defined as follows for all $x\in \cX:$
        \[
        h(x) = \Maj\left(\{h_1, \dots, h_t\}\right)(x).
        \]

    \end{enumerate}
    \end{tcolorbox}
    \caption{$\SPV$ - A meta algorithm implementing a stable version of the input learning algorithm $\Learn$.}
    \label{fig:SPV}
\end{figure}

To state the complementing impossibility result, we need the following definition of \emph{non-trivial concept classes} \citep*{bshouty2002pac}.

\begin{definition}[Non-trivial concept classes]\label{def:non-trivial}  We say that a concept class $\cH$  over a domain $\cX$ is \emph{non-trivial}, if there are $x_1,x_2\in \cX$ and $h_1,h_2 \in \cH$ so that $h_1(x_1)=h_2(x_1)$ and $h_1(x_2) \neq h_2(x_2)$.
\end{definition}

\begin{theorem}[Realizable case -- impossibility result] \label{theo:neg}
There exists a constant $c_2 > 0$ so that the following holds. Let $\cH$ be a non-trivial hypothesis class with VC dimension $d$ and let $\eta\in (0,1)$. Then, there exists a distribution $D$ realizable by $\cH$, so that every learner $\Learn$ has $\eta$-adversarial risk
\[
\eps^{\Adv}_n(\Learn|D, \eta) \geq \min\{c_2 \eta d, 1/100\}
\]
for any sample size $n \geq 1/\eta$.
%\textcolor{red}{Do we really want to state ``(deterministic)''? after all this theorem applies to randomized learners as well and we just make it sound weaker to readers who will not read between the lines.}
%\textcolor{red}{We should explain why the lower and upper bound match. It would be best if we could phrase them in a way that its obvious (the disjunction between the conditions $\eps,\eta$ need to satisfy for the lower bound is confusing in this regard.)}
\end{theorem}
We note that this impossibility result applies also to a variety of randomized learners; we elaborate on this in Appendix~\ref{apndx:neg_proof},
where we also prove Theorem~\ref{theo:neg}.

The above lower bound demonstrates how vulnerability to instance-targeted attacks depends greatly on the hypothesis class we want to learn, and specifically on its VC-dimension.

\subsection{Certification}

Besides prediction-stability, another useful property our $\SPV$ meta-algorithm has is the ability to efficiently calculate and output a \emph{certificate} for the stability of its predictions. Formally, given an input sample $S$, a certificate is a function $\eta_S: \cX \rightarrow [0,1]$, outputted by a learner in addition to its output hypothesis $h_S$ such that the following is satisfied: $h_S(x)=h_{S'}(x)$
for every point $x$ and for every input sample $S'$ which is $\eta_S(x)$-close to $S$.
If one ignores computational considerations, outputting optimal certificates is always possible:

\begin{definition}[Optimal Certificate]
Let $\Learn$ be any learning rule, and let $S$ be an input sample.
Define the \emph{optimal certificate} $\eta^\star(\cdot)= \eta^\star(\cdot \vert S)$ of $\Learn$ on input sample $S$  as follows. The optimal certificate $\eta^\star(x \vert S)$ is equal to $\frac{k}{n}$ where $k$ is the largest integer for which $\Learn(S')(x)= \Learn(S)(x)$ for every sample $S'$ with hamming distance at most $k$ from $S$.
% To add an optimal certificate $\eta(x)$ to the output $y$, the learner sets $\eta(x)=0$, and then for all $S'\in B_{ \eta(x) + 1/n }(S)$ (where $n = |S|$) checks whether $\Learn(S')(x) = y$. If this condition does not hold, output $\eta(x)$ as a certificate. If it holds, set $\eta(x) := \eta(x) + 1/n$ and repeat the previous step.
\end{definition}

% In other words, $\eta^\star(\cdot)$ captures the smallest amount of noise that an adversary would need 
%     in order to flip the the prediction of $\Learn$ on $x$. 
In other words, if $S$ is a sample that was corrupted by an adversary with budget $\eta$ such that $\eta \leq  \eta^\star(x \vert S)$
    then the output label $\Learn(S)(x)$ is equal to the label that would have been outputted if the learner was trained with the uncorrupted sample.

The issue with the optimal certificate $\eta^\star(x)$ is that it can be impossible to compute
    as it requires to iterate over the potentially infinite space of all samples $S'$ of hamming distance at most $n\cdot \eta(x)$ from the input sample $S$. 
    In contrast, our $\SPV$ learner can \emph{efficiently} calculate a non-trivial {lower bound} on~$\eta^\star$ which therefore
    also serves as a certificate. The key property which enables this is the fact that its output hypothesis is the majority vote of base learners, each trained on a \emph{disjoint} subsample. This is summarized in the following proposition:

% Our $\SPV$ learner, in contrast with the above remark that admits no reasonable implementation of certifying, allows us to practically calculate a non-trivial certificate. This can be easily achieved by training learners on disjoint input samples and then returning the majority vote. 

\begin{proposition} \label{prop:cert}
Consider a learner whose output hypothesis is given by a majority vote of $t$ learners $L_1, \dots, L_t$ that are trained on $t$ disjoint subsamples $S_1, \dots, S_t$ of the input sample $S$. Define
\[
\eta(x\vert S) = \frac{1}{n}\cdot \Bigl(\frac{\sum_{i\in [t]} 1[h_i(x) = y] - \sum_{i\in [t]} 1[y_i \neq y]}{2} - 1\Bigr),
%\textcolor{red}{\eta(x\vert S) = \max \left\{0, \frac{1}{n} \left(\left(\sum_{i\in [t]} 1[y_i = y]\right) -1-t/2 \right) \right\},}
\]
where $h_i$ is the output hypothesis of $L_i$, $y$ is the output label of the majority vote of the $L_i$'s, and $n$ is the size of the input sample $S$. Then, $\eta(x \vert S)\leq \eta^\star(x \vert S)$.
\end{proposition}

\begin{proof}
Notice that $n\cdot \eta(x \vert S)+1$ is equal to the minimal number of $h_i$'s whose prediction on $x$ must be flipped in order to enforce that 
\[\lvert \{i : h_i(x)=y\}\rvert \leq \lvert \{i : h_i(x)\neq y\}\rvert.\]
 Therefore, at least one example in each $S_i$ such that $h_i(x)=y$ must be replaced in order to change the prediction of $\Learn(S)$ on $x$. In particular, if only $n\cdot \eta(x \vert S)$ examples are replaced than the prediction of $\Learn(S)$ on $x$ remains the same. This implies that $\eta^\star(x \vert S)\geq \eta(x \vert S)$ as stated.
% If $\eta(x)=0$, it is trivial that the certificate $\eta(x)$ is always valid. Otherwise, take an arbitrary $S'\in B_{\eta(x)}(S)$. Since at most $\eta(x) n$ examples are different between $S$ and $S'$, it holds that $S_i \neq S'_i$ for at most $\eta(x) n$ many indices $i \in [t]$. So, at most $\eta(x) n$ learners $L_i$ that predicted $y$ when given $S_i$ will predict $1-y$ when given $S'_i$. Thus, at least
% \[
% \sum_{i\in [t]} \left[1[y_i = y]\right] - \eta n = t/2 + 1
% \]
% learners will predict the label $y$ even when given $S'_i$ instead of $S_i$. The majority vote will thus still determine the same label $y$, even when the input sample is changed from $S$ to $S'$. 
\end{proof}

In light of Proposition~\ref{prop:cert}, our $\SPV$ learner can efficiently compute and output a certificate $\eta(x)$ which is proportional to $\eta$ (where $\eta$ is the stability parameter given to $\SPV$), with probability proportional to the expected loss of the input learner given to $\SPV$ when executed on a sample of size $\left \lceil \frac{1}{7 \eta} \right \rceil$.

\subsection{A Proper Variant of $\SPV$}

\begin{figure}
    \centering
    \begin{tcolorbox}
    \begin{center}
        $\PSPV$: \textsc{Proper Stable Partition and Vote}
    \end{center}
    \textbf{Input:} Stability parameter $\eta \in (0,1)$, a proper learning algorithm $\Learn_p$ and an input sample $S \sim D^n$ where $n \geq 1/\eta$.\\
    \textbf{Output:} A classifier $h\in \cH$.
    \begin{enumerate}
        \item Partition $S$ into $\lceil 5 k_p \eta n \rceil$ consecutive subsamples such that all first $t = \lfloor 5 k_p \eta n \rfloor$ subsamples are of size at least $\frac{1}{5 k_p \eta}$. Denote the $i$'th subsample by $S^{(i)}$.
        \item For all $i\in [t]$, train $\Learn_p$ on $S^{(i)}$ to obtain a hypothesis $h_i = \Learn_p(S^{(i)})$.
        \item Return $h\in \cH$ such that 
        \[
        h(x) 
        = 
        \Maj \left(\left\{h_1, \dots, h_t\right\}\right)(x)
        \]
        holds for all $x\in \cX_{\{h_1, \dots, h_t\}, 2k_p}$.
        %\Mnote{this is not a poly-time algorithm. can we make it so? eg. what if we just want agreement empirically on the training set?} 
        %\textcolor{red}{Shay: The way the lase item is written suggests that such an $h$ always exists, which is not the case.}
    \end{enumerate}
    \end{tcolorbox}
    \caption{$\PSPV$ - A meta-algorithm that implements a stable version of the input proper learning algorithm $\Learn_p$ and maintains properness.}
    \label{fig:PSPV}
\end{figure}

%\textcolor{red}{Shay: add the result of $\eps=O_\cH(\sqrt{\eta})$ which holds for \emph{every} VC-class.} \Inote{I think Mohammad said that this is a known result.}
%\textcolor{red}{Reference?} \Mnote{Theorem 3.3 from \url{https://arxiv.org/pdf/2105.08709.pdf} is the sub-sampling construction. it only works in the weak setting.}
%\textcolor{red}{Shay: thanks Mohammad, I looked at the reference and agree that it *follows* as a corollary from Theorem 3.3 there. However, this theorem does not focus on the dependence between $\eta$ and $\eps$ as we do here, and there are many other details there. Thus, I think we should explicitly state this in this work. In order to remain self-contained I think we should also include the simple proof. (But of couse note that it follows from that work.)}
%\Mnote{I agree that the exact dependence is not discussed there and we can discuss that aspect here. The main issue is that the theorem is in the weak setting. Do we want to expand the setting to weak as well?}

%In this section we describe the $\PSPV$ meta-algorithm, described in Figure~\ref{fig:PSPV}, which is a proper version of $\SPV$. We analyze its general performances, and show that for the class of linear classifiers it can perform reasonably well, when fed with the SVM learner.

%\begin{remark}
%By Item~\ref{obs:eqv_rob_stab_itm1} in Observation~\ref{obs:eqv_rob_stab}, it is implied that if $\frac{\min\{\eps, \sigma\}}{\eta} \geq c\cdot d^3$, then $\cH$ is proper-$(\eps+\sigma)$-learnable under $\eta$-targeted poisoning attacks. 
%\end{remark}

We now present a proper version of $\SPV$ for classes $\cH$ with a finite \emph{projection number}, described in Figure~\ref{fig:PSPV}.
The projection number of a concept class $\cH$ is denoted by $k_p=k_p(\cH)$ (we present its definition after the statement of Theorem~\ref{theo:proper} below).
In particular, for the class of halfspaces it yields a robust learner with the following guarantee:
\begin{theorem} \label{theo:proper}
There exists a constant $c > 0$ so that the following holds. Let $\cH$ be the class of halfspaces over $\mathbb{R}^d$ for some $d \geq 1$, and let $\eta \in (0,1)$. Then, there exists a proper learner $\Learn$ having $\eta$-adversarial risk
\[
\eps_{n}(\Learn | D, \eta) \leq c \eta d^3
\]
for any distribution $D$ realizable by $\cH$ and for any sample size $n \geq 1/\eta$.
\end{theorem}
The proof of Theorem~\ref{theo:proper} is deferred to Appendix~\ref{apndx:proper}.

To derive Theorem~\ref{theo:proper}, we reinforce the $\SPV$ algorithm with a technique
introduced by~\citet*{kane2019communication} and further developed by \citet*{bousquet2020proper}.
This technique allows in certain cases to \emph{project} a majority vote of hypotheses from the class  $\cH$ back to $\cH$. 
Its applicability hinges on a combinatorial parameter called the \emph{projection number}. The $\PSPV$ learner explicitly uses the projection number, so for completeness we give its definition below. The interested may see the work of \citet*{bousquet2020proper} for an insightful discussion on the role of the projection number in proper learning. 

\begin{definition}[Projection Number]
Let $\cH$ be a concept class. For any $\ell\geq 2$ and for any multiset $\cH' \subset \cH$ define the set $\cX_{\cH',\ell}$ to be the set of all $x\in \cX$, for which the number of hypotheses in $\cH'$ that disagree with $\Maj(\cH')(x)$ is less than $\lvert \cH' \rvert/\ell$. The Projection Number of the class $\cH$, denoted $k_p=k_p(\cH)$, is defined to be the smallest $\ell$ so that for any finite multiset $\cH' \subset \cH$, there exist $h\in \cH$ such that $h(x) = \Maj(\cH')(x)$ for all $x\in \cX_{\cH',\ell}$.  If no such $\ell$ exists then $k_p = \infty$.
\end{definition}

%\shay{I don't see the point in including this definition here. It is useless for readers who know it and will not be clear to readers who don't. Either add a discussion or move it to the appendix. } \Inote{The reason to include the definition is because the algorithm heavily relies on it, and we decided to include the algorithm in the main body. The usage is not even in a black-box manner - we use the more internal definition of the set $\cX_{\cH', \ell}$ in the algorithm. Without the definition, the algorithm seems kind of incomplete. I referred to the relevant paper for further discussion.}

%\Inote{Write before the projection number definition, that PSPV uses a technical term of projection nubmer that we use in the algorithm and therefore we give the definition here.}

\section{Agnostic Setting}

%\textcolor{red}{Shay: This part should come after the results in the realizable case, and (as we discussed today) the story here should be guided by the question which guarantees on the optimal $\eps$ as a function of $\eta$ and the class $\cH$ are possible in the agnostic case. We can begin by explaining Steve's argument that $O(d(\OPT +\eta)$ is possible by reduction to the realizable setting and then raise the natural question whether $\OPT + O(d\eta)$ is possible like in typical agnostic settings where the goal is to get an additive vanishing regret. Then we can state the lower bound asserting that the latter is impossible and conclude with the upper bound.}
%\Mnote{Since the proper results are for realizable so far, we can probably keep the order as is, but I will add a discussion about the reduction.}

In this section, we extend the results on robust learnability to the agnostic case. First, by a simple generalization of the positive result for the realizable case, we provide a robust semi-agnostic learner. That is, our learner has adversarial risk depending linearly on $\OPT = \OPT(\cH,D) := \min_{h \in \cH} L_D(h)$. While semi-agnostic learning is considered not ideal in many cases, we complement our positive result by showing that semi-agnostic learning is unavoidable when the goal is to design a robust and deterministic (as ours) learner for the agnostic setting.

\subsection{A Semi-agnostic Learner}

Formally, a semi agnostic learner is defined as follows. Let $c \in \mathbb{R}$. A learning rule $\Learn$ is a \emph{$c$-semi agnostic learner} if the following holds. Let $\cH$ be a concept class and let $D$ be a distribution over examples. Then there exists an \emph{excess error rate} $\eps^{\Agn}: \mathbb{N} \rightarrow [0,1]$ such that $\eps_n(\Learn \vert D) \leq c \OPT + \eps^{\Agn}(n)$ where $\OPT = \inf_{h \in \cH} L_D(h)$.

Before stating our positive result in this setting, we first discuss how achieving adversarial risk $O(d(\OPT +\eta)$ is possible by reduction to the realizable setting. 

\textbf{Reduction to the realizable setting.} Suppose a learner is given a training set $S'$ of size $n$ that comes with $\eta n$ replacements made by the adversary on the original set $S$. Moreover, suppose that $S$ is sampled from a distribution $D$ such that the best $h \in \cH$ has $\OPT$ error on $D$. This means that, roughly $\OPT$ fraction of $S$ does \emph{not} match to $h$. Therefore, one can see $S'$  as first sampled  from $D$ (without noise) followed by $\approx (\eta + \OPT) \cdot n$ replacement corruptions. This way, one can employ a learner that can tolerate $\eta'=\eta + \OPT$ fraction of adversarial corruptions in the realizable setting  and obtain total adversarial risk $O(d(\OPT +\eta)$.

The above discussion raises a natural question: can a learner  achieve adversarial risk $O(\OPT + d\eta)$ or even (ideally) $\OPT + O(d \eta)$? The latter is the  typical type of risk bound in agnostic settings, where there is no multiplicative dependence on $\OPT$ in the risk.

The following theorem, which we prove in Appendix~\ref{apndx:agn_pos} states the positive result.

\begin{theorem}[Positive result for the agnostic case] \label{theo:pos_semi_agn}
There exist constants $c_1,c_2$ so that the following holds. Let $\cH$ be a hypothesis class with VC dimension $d$ and let $\eta\in (0,1)$. 
Then, there exists a learner $\Learn$ having $\eta$-adversarial risk 
\[\eps^{\Adv}_n(\Learn | D, \eta) \leq c_2\cdot\OPT + c_1\cdot d\cdot \eta\] 
for any distribution $D$ over examples and for any sample size $n \geq 1/\eta$.
%for any (not necessarily realizable) distribution $D$ over examples, and
% for all $(\eps, \eta)\in (0,1)^2$ satisfying $\eps \geq c_1 \eta d$
% %\begin{align*}
% %\frac{\min\{\eps, \sigma\}}{\eta} \geq c\cdot d &\implies (\eps, \sigma, \eta) \in \APSL.
% %\end{align*}
% %$\cH$ is semi-agnostic $(\eps, \eta)$-robustly learnable.
% there exists a learner $\Learn$ having $\eta$-adversarial risk $\eps_n(\Learn | D, \eta) \leq c_2 \OPT + \eps$ for any distribution $D$ over examples and for any sample size $n \geq 1/\eta$.

% There exist constants $c_1,c_2$ so that the following holds. Let $\cH$ be a hypothesis class with VC dimension $d$. Then,
% %for any (not necessarily realizable) distribution $D$ over examples, and
% for all $(\eps, \eta)\in (0,1)^2$ satisfying $\eps \geq c_1 \eta d$
% %\begin{align*}
% %\frac{\min\{\eps, \sigma\}}{\eta} \geq c\cdot d &\implies (\eps, \sigma, \eta) \in \APSL.
% %\end{align*}
% %$\cH$ is semi-agnostic $(\eps, \eta)$-robustly learnable.
% there exists a learner $\Learn$ having $\eta$-adversarial risk $\eps_n(\Learn | D, \eta) \leq c_2 \OPT + \eps$ for any distribution $D$ over examples and for any sample size $n \geq 1/\eta$.
\end{theorem}

As in the realizable case upper bound, the above upper bound is proved by using the $\SPV$ meta-learner. The main difference is that to prove this result we use a different input learner $\Learn$ given to $\SPV$ than the one we use in the realizable case.

\subsection{Ruling Out Agnostic Learning}

Note that Theorem~\ref{theo:pos_semi_agn} only proves the existence of a \emph{semi}-agnostic learner under instance-targeted poisoning. A more desirable goal would be to obtain (standard) \emph{agnostic} learners whose error under the attack is $ \OPT + \psi$ where $\psi$ is a vanishing (additive) error term when $\eta \to 0$. Here we will prove that at least when it comes to \emph{deterministic} learners, such a goal is out of reach, and the best we can hope for is $2 \OPT$ plus additive terms that depend on $\eta$ and the VC dimension. This explains why we can only achieve a semi-agnostic learner.

The following theorem, which we prove in Appendix~\ref{apndx:agn_neg} shows that in Theorem~\ref{theo:pos_semi_agn}, the constant $c_1$ needs to be at least  $2$, and so the standard way agnostic learners bound their regret is not possible for instance-targeted poisoning.
\begin{theorem}[Impossibility of agnostic learning] \label{theo:agn_neg}
Let $\eta' \in (0,1), n\in \mathbb{N}$. For any hypothesis class $\cH$ that has at least two hypotheses   and for any deterministic learner, there is a distribution $D$ over (two) examples and $\eta=\eta'+\widetilde{O}(1/\sqrt{n})$ such that $\Learn$ has $\eta$-adversarial risk
\[
\eps^{\Adv}_n(\Learn |D , \eta) \geq 2\OPT+\Omega(\eta')-O(1/n).
\]
\end{theorem}

\section{Conclusion and Open Questions} \label{sec:conc}
In this work, we studied the optimal rate of learning for binary classification problems under instance-targeted poisoning. We showed that in the realizable setting the error rate can be characterized up to a constant factor and is proportional both to adversary's budget and the VC dimension of the class. In the agnostic setting, we proved a perhaps surprising lower bound that standard agnostic learning (with additive regret compared to the optimal error in the no-attack setting) is impossible for deterministic learners, and also complemented this with a positive result using a \emph{semi-agnostic} learner. We also showed how to make our learners proper in a variety of interesting settings.

Our work leaves a few interesting directions for future research.
\begin{itemize}
    \item \textbf{Finding the exact constant in the realizable case.}
Our results in the realizable case characterize the optimal adversarial risk up to a constant multiplicative factor in the sense that there exist constants $c_1,c_2$ so that achieving $\eta$-adversarial risk of $c_1 \eta d$ is possible for any hypothesis class with VC-dimension $d$, whereas obtaining $\eta$-adversarial risk of $c_2 \eta d$ can't be achieved for any hypothesis class with VC-dimension $d$. However, there is a large gap between $c_1,c_2$. Can we close or shrink this gap?
%Concretely, $c_1 >1$ and $c_2 <1$. Can one reproduce our results with $c_1<1$ or with $c_2 > 1$?

  \item \textbf{Finding the correct multiplicative factor in the agnostic case.}
Our results show that in the agnostic case, there must be a constant $C\geq 2$ so that the best adversarial risk attainable is $C \cdot \OPT$. What is the value of $C$?

  \item \textbf{Characterizing proper robust learning.}
In the proper and realizable case, our stable learner for linear classifiers depends on $d^3$, while our lower bound depends linearly on $d$, as in the general improper case. It remains open to identify the correct dependence on $d$.

  \item \textbf{Characterizing the role of randomness.} Our impossibility result for the agnostic learning (Theorem~\ref{theo:agn_neg}) only applies to deterministic learners. It remains open to either effectively use randomness during the learning (known or unknown to the adversary) and obtain an agnostic learner, or to extend the negative result to cover such randomized learners as well.
\end{itemize}

\section*{Acknowledgments}
Amin Karbasi acknowledges funding in direct support of this work from NSF (IIS-1845032), ONR (N00014- 19-1-2406), and the AI Institute for Learning-Enabled Optimization at Scale (TILOS). Mohammad Mahmoody is supported by NSF grants CCF-1910681 and CNS1936799. Shay Moran is a Robert J.\ Shillman Fellow; he acknowledges support by ERC grant 802599, by ISF grant 1225/20, by BSF grant 2018385, by an Azrieli Faculty Fellowship, by Israel PBC-VATAT, and by the Technion Center for Machine Learning and Intelligent Systems (MLIS). We thank anonymous NeurIPS 2022 reviewers for helping us to improve this paper, and for pointing out good motivating examples. 

\bibliographystyle{plainnat}
\bibliography{bibl,Biblio/OtherRefs}

\newpage

\remove{
\section*{Checklist}

%%% BEGIN INSTRUCTIONS %%%
%The checklist follows the references.  Please
%read the checklist guidelines carefully for information on how to answer these
%questions.  For each question, change the default \answerTODO{} to \answerYes{},
%\answerNo{}, or \answerNA{}.  You are strongly encouraged to include a {\bf
%justification to your answer}, either by referencing the appropriate section of
%your paper or providing a brief inline description.  For example:
%\begin{itemize}
%  \item Did you include the license to the code and datasets? \answerYes{See Section~\ref{gen_inst}.}
%  \item Did you include the license to the code and datasets? \answerNo{The code and the data are proprietary.}
%  \item Did you include the license to the code and datasets? \answerNA{}
%\end{itemize}
%Please do not modify the questions and only use the provided macros for your
%answers.  Note that the Checklist section does not count towards the page
%limit.  In your paper, please delete this instructions block and only keep the
%Checklist section heading above along with the questions/answers below.
%%% END INSTRUCTIONS %%%

\begin{enumerate}

\item For all authors...
\begin{enumerate}
  \item Do the main claims made in the abstract and introduction accurately reflect the paper's contributions and scope?
    \answerYes
  \item Did you describe the limitations of your work?
    \answerYes
  \item Did you discuss any potential negative societal impacts of your work?
    \answerNo
  \item Have you read the ethics review guidelines and ensured that your paper conforms to them?
    \answerYes
\end{enumerate}

\item If you are including theoretical results...
\begin{enumerate}
  \item Did you state the full set of assumptions of all theoretical results?
    \answerYes
        \item Did you include complete proofs of all theoretical results?
    \answerYes
\end{enumerate}

\item If you ran experiments...
\begin{enumerate}
  \item Did you include the code, data, and instructions needed to reproduce the main experimental results (either in the supplemental material or as a URL)?
    \answerNA
  \item Did you specify all the training details (e.g., data splits, hyperparameters, how they were chosen)?
    \answerNA
        \item Did you report error bars (e.g., with respect to the random seed after running experiments multiple times)?
    \answerNA
        \item Did you include the total amount of compute and the type of resources used (e.g., type of GPUs, internal cluster, or cloud provider)?
    \answerNA
\end{enumerate}

\item If you are using existing assets (e.g., code, data, models) or curating/releasing new assets...
\begin{enumerate}
  \item If your work uses existing assets, did you cite the creators?
    \answerNA
  \item Did you mention the license of the assets?
    \answerNA
  \item Did you include any new assets either in the supplemental material or as a URL?
    \answerNA
  \item Did you discuss whether and how consent was obtained from people whose data you're using/curating?
    \answerNA
  \item Did you discuss whether the data you are using/curating contains personally identifiable information or offensive content?
    \answerNA
\end{enumerate}

\item If you used crowdsourcing or conducted research with human subjects...
\begin{enumerate}
  \item Did you include the full text of instructions given to participants and screenshots, if applicable?
    \answerNA
  \item Did you describe any potential participant risks, with links to Institutional Review Board (IRB) approvals, if applicable?
    \answerNA
  \item Did you include the estimated hourly wage paid to participants and the total amount spent on participant compensation?
    \answerNA
\end{enumerate}

\end{enumerate}

}

\newpage

\appendix

%\section{The One-inclusion graph learner}

%\subsection{Algorithm overview}

%\subsection{A semi-agnostic variation overview}

\section*{Supplementary  Material}

\section{Proof of Theorem~\ref{theo:pos} (Realizable Case -- Positive Result)} \label{sec:pos_proof}

\begin{theorem*}[Restatement of Theorem~\ref{theo:pos}]
There exists a constant $c_1 > 0$ so that the following holds. Let $\cH$ be a hypothesis class with VC dimension $d$ and let $\eta\in (0,1)$.
Then there exists a learner $\Learn$ having $\eta$-adversarial risk
\[
\eps_{n}^{\Adv}(\Learn | D, \eta) \leq c_1 \eta d
\]
for any distribution $D$ realizable by $\cH$ and for any sample size $n \geq 1/\eta$.
\end{theorem*}

To prove Theorem~\ref{theo:pos}, we will use the \textsc{Stable Partition and Vote} (or $\SPV$ for short) meta learner described in Figure~\ref{fig:SPV} with the One-inclusion graph algorithm of \citet*{haussler1994predicting} as the input learner. 
% The resulting algorithm is very much based on a learner suggested by \citet{gao2021learning} (which in turn builds upon ideas from the work of \citet{levine2020deep}). The main differences are that (1) we partition the input sample to a different number of slices $t$, and (2) we train the One-inclusion graph learner on each slice. 
First, we prove a more general result on the performance of our $\SPV$ meta learner. We denote the algorithm obtained by executing $\SPV$ with a learner $\Learn$ as the input algorithm by $\SPV(\Learn)$.

\begin{lemma}[General performance of $\SPV$] \label{lem:SPV_perf}
Let $\cH$ be a concept class, $D$ be a distribution over examples, and $\Learn$ be a learning rule. Let also $\eta \in(0,1)$ be the stability parameter given to $\SPV$ and let $n\geq 1/\eta$ be the sample size. Then $\SPV(\Learn)$ has $\eta$-adversarial risk
\[
\eps^{\Adv}_n(\SPV(\Learn) | D, \eta) \leq 6\eps_{ \lceil 1/(7\eta)  \rceil}(\Learn | D).
\]
Recall that $\eps_{ \lceil 1/(7\eta)  \rceil}(\Learn | D)$ is the expected population loss of $\Learn$
when trained on a sample of size $\lceil 1/(7\eta)  \rceil$ from $D$ (in the standard, non adversarial, setting).
\end{lemma}

\begin{proof}
Let $S\sim D^n$ be the input sample, and $(x,y)\sim D$ be the test example. Note that for all $i\in [t]$ (where $t= \lfloor 7 \eta n \rfloor$ is the number of subsamples of size at least $\frac{1}{7 \eta}$ in the partition made by $\SPV$) it holds that $\E\bigl[1[h_i(x) \neq y]\bigr] \leq \eps_{\lceil 1/(7\eta) \rceil}(\Learn | D)$.
By applying linearity of expectation we get

\[\E \left[ \frac{1}{t} \sum_{i=1}^t 1[h_i(x) \neq y] \right] \leq \eps_{\lceil 1/(7\eta) \rceil}(\Learn | D).
\]

By Markov's inequality:

\[
\Pr \left[\frac{1}{t} \sum_{i=1}^t 1[h_i(x) \neq y] \geq 1/6 \right] \leq 6 \eps_{\lceil 1/(7\eta) \rceil}(\Learn | D).
\]

Let $S' \in B_\eta (S)$. Let $h'= \SPV(\Learn)(S')$, and for all $i\in [t]$ let $h'_i$ be the hypothesis obtained by training $\Learn$ on $S'^{(i)}$. Note that, since $S$ and $S'$ are $\eta$-close by, and since $n \geq 1/\eta$ it holds that
\[
\frac{1}{t} \sum_{i=1}^t 1\left[S^{(i)} \neq S'^{(i)}\right] \leq \frac{\eta n}{ \lfloor 7 \eta n \rfloor} \leq 1/6.
\]
Hence it is implied that $\frac{1}{t} \sum_{i=1}^t 1 \left[h_i(x) \neq h'_i(x)\right] \leq 1/6$. 
Thus, the event that $\frac{1}{t} \sum_{i=1}^t 1[h'_i(x) \neq y] \geq 1/3$ implies (or, is contained in) the event that $\frac{1}{t} \sum_{i=1}^t 1[h_i(x) \neq y] \geq 1/6$, hence,

% Thus, in the event where $\frac{1}{t} \sum_{i=1}^t 1[h_i(x) \neq y] < 1/6$, it also holds that $\frac{1}{t} \sum_{i=1}^t 1[h'_i(x) \neq y] < 1/3$.  Combined with Equation~\eqref{eq:pos}, this implies
\[
\Pr \left[ \frac{1}{t}\sum_{i=1}^t 1[h'_i(x) \neq y] \geq 1/3 \right] 
\leq
6\eps_{\lceil 1/(7\eta) \rceil}(\Learn | D).
\]
Since $h'(x)$ is a majority vote of $\{h'_1(x), \dots, h'_t(x)\}$, the above implies that
\[
\Pr[h'(x) \neq y] \leq  6\eps_{\lceil 1/(7\eta) \rceil}(\Learn | D). 
\]
Since $S'$ is an arbitrary sample in $B_{\eta}(S)$, the above implies that $\SPV(\Learn)$ has the stated $\eta$-adversarial risk.
\end{proof}

To prove Theorem~\ref{theo:pos}, we will need an optimal learner as an input learner for $\SPV$.

\begin{theorem}[\citet*{haussler1994predicting}] \label{theo:one_inc}
Let $\cH$ be a concept class with VC-dimension $d$, and let $D$ be a distribution realizable by $\cH$. Let also $n\in \mathbb{N}$, and let $\Learn$ be the One-inclusion graph algorithm. Then $\eps_n(\Learn | D) \leq \frac{d}{n+1}$.
%for any distribution $D$ over examples which is realizable by $\cH$, if $S \sim D^n$ then $L_D(\Learn(S)) \leq \frac{d}{n+1}$.
\end{theorem}

Theorem~\ref{theo:pos} can now be immediately inferred as a direct application of Lemma~\ref{lem:SPV_perf} and Theorem~\ref{theo:one_inc}.

\begin{corollary}[Realizable case -- positive result]
Let $\cH$ be a concept class with VC-dimension $d$, let $D$ be a distribution realizable by $\cH$, and let $\Learn$ be the One-inclusion graph algorithm. Let also $\eta \in (0,1)$ be the stability parameter given to $\SPV$ and let $n\geq 1/\eta$ be the sample size. Then $\SPV(\Learn)$ has $\eta$-adversarial risk

\[
\eps^{\Adv}_n(\SPV(\Learn) | D, \eta) \leq 42 \eta d.
\]
\end{corollary}

\begin{proof}
By Theorem~\ref{theo:one_inc}, plug in $\eps_{\lceil 1/(7\eta) \rceil}(\Learn | D) \leq \frac{d}{\lceil 1/(7\eta) \rceil + 1} \leq 7 \eta d$ to Lemma~\ref{lem:SPV_perf} and the result follows.
\end{proof}

\section{Proof of Theorem~\ref{theo:neg} (Realizable Case -- Impossibility Result)} \label{apndx:neg_proof}

\textbf{Randomized Learning Rules.}
The impossibility result in Theorem~\ref{theo:neg} extends to randomized learning rules.
But in order for the statement in Theorem~\ref{theo:neg} to be meaningful, we need to define adversarial risk with respect to randomized learners.
As common in the literature on learning theory (see, e.g.\ the book of ~\citet*{shalev2014understanding}) we model randomized learners
as deterministic learning rules with \emph{continuous } predictions $p\in[0,1]$, and loss function $\ell(p,y)= \lvert p -y\rvert$. Indeed, the loss of a deterministic learner predicting a value $p \in [0,1]$ under the loss function $\lvert y-p\rvert$ is equal to the expected $0/1$-loss of a randomized learner predicting $1$ with probability $p$. In the course of discussing the impossibility result, a \emph{learning algorithm} $\Learn \colon (\cX \times \{0,1\})^* \rightarrow [0,1]^{\cX}$ is a deterministic mapping which takes an input sample $S\in (\cX \times \{0,1\})^*$ and maps it to a hypothesis $f\in [0,1]^{\cX}$. We re-define $\eta$-adversarial risk with this view of randomized learners as \emph{randomized $\eta$-adversarial risk}.

%\begin{definition}[Randomized risk]
%The risk of a learning algorithm $\Learn$ that outputs predictions from $[0,1]$ w.r.t a distribution $D$ and a sample size $n$ is defined by

%\[
%\eps_n(\Learn | D) := \mathbb{E}_{S \sim D^n, (x,y) \sim D}[\lvert \Learn(S)(x) - y \rvert].
%\]
%\end{definition}

\begin{definition}[Randomized $\eta$-Adversarial Risk] \label{def:adv-risk-rand}
Let $\eta\in(0,1)$ be the adversaries' budget, let $\Learn$ be a learning rule, 
and let $D$ be a distribution over examples. The \emph{randomized $\eta$-adversarial risk} of $\Learn$ w.r.t $D$ and sample size $n$ is defined by
\[
\eps^{\Adv}_n(\Learn \vert D,\eta):=
\mathbb{E}_{S \sim D^n, (x,y) \sim D} \left[ \sup_{S' \in B_{\eta}(S)} \lvert \Learn(S')(x) - y \rvert \right].
\]
%\shay{I find the superscript $\rand$ redundant as it can be inferred from the context. I did get confuse however from the fact that the adversarial loss and the standard loss are both denoted by $\eps_n$ and the only difference between them is $\eta$.} \Inote{I thought that it is more understandable since otherwise the notation is exactly the same as in previous definition.. but I really don't mind removing it if you think it's better. We can also change $\eps_n$ to $\eps_n^{\Adv}$ in case of adversarial loss, if you find it better.}
%\shay{Yes, please change it: the distinction between randomized and deterministic is less central than the distinction between adversarial and standard loss.}
\end{definition}

%The definition of Robust learnability stays the same w.r.t to the above definition.
The above definition of adversarial risk captures the case of an adversary that knows the expected prediction of the learner (that is, its test-time randomness), but not the learner's "internal" randomness (computation-time randomness). Indeed, the supremum is taken only with respect to the expected prediction, and not with respect to a specific execution of the algorithm determined by its internal randomness.
%\Mnote{I don't think this is correct.}
Note that deterministic learners are a 
special case ($\{0,1\}$-valued outputs), in which case this definition collapses to the previous Definition~\ref{def:risk}.
To avoid further notation, note that we overloaded the notation $\eps^{\Adv}_n$ from Definition~\ref{def:risk} in the above more general definition.   

We are now ready to prove the impossibility result.
%\shay{Restate the Theorem here (and in general restate theorems in the appendix).}

\begin{theorem*}[Restatement of Theorem~\ref{theo:neg}]
There exists a constant $c_2 > 0$ so that the following holds. Let $\cH$ be a non-trivial hypothesis class with VC dimension $d$ and let $\eta\in (0,1)$. Then, there exists a distribution $D$ realizable by $\cH$, so that every learner $\Learn$ has $\eta$-adversarial risk
\[
\eps^{\Adv}_n(\Learn|D, \eta) \geq \min\{c_2 \eta d, 1/100\}
\]
for any sample size $n \geq 1/\eta$.
%\textcolor{red}{Do we really want to state ``(deterministic)''? after all this theorem applies to randomized learners as well and we just make it sound weaker to readers who will not read between the lines.}
%\textcolor{red}{We should explain why the lower and upper bound match. It would be best if we could phrase them in a way that its obvious (the disjunction between the conditions $\eps,\eta$ need to satisfy for the lower bound is confusing in this regard.)}
\end{theorem*}

\begin{proof}

Let $\cH$ be a non-trivial concept class; in particular this means that its VC-dimension $d$ satisfies $d\geq 1$. 
Let $\eta \in (0,1)$ be the adversaries' budget and let $\Learn$ be an arbitrary learner. 
We need to show that there exists a distribution~$D$ realizable by $\cH$ so that $\eps^{\Adv}_n(\Learn \vert D,\eta) \geq \min\{ \eta d / 32, 1/100\}$. 

It suffices to consider the case when $\eta d /32 \leq 1/100$ and prove that $\eps^{\Adv}_n(\Learn \vert D,\eta) \geq \eta d / 32$. 
Indeed, in the complementing case we have $\eta d /32 > 1/100$ and we need to show that $\eps^{\Adv}_n(\Learn \vert D,\eta) \geq 1/100$. 
Notice that $\eta d /32 > 1/100$ is equivalent to $\eta > \frac{32}{100d}$, and thus it suffices to show that even if the adversary's budget $\eta$ is reduced to $\eta = \frac{32}{100d}$ then $\eps^{\Adv}_n(\Learn \vert D,\eta) \geq 1/100$. The latter indeed follows from the case when $\eta d /32 \leq 1/100$, because $\frac{32}{100d} \cdot d /32 = 1/100$.

We thus assume that $\eta d /32 \leq 1/100$ and set out to prove that $\eps^{\Adv}_n(\Learn \vert D,\eta) \geq \eta d / 32$. 
We first consider the case when the VC-dimension of $\cH$ is $d \geq 2$ and later handle the case when $d=1$.

% It suffices to assume that $\eta d /32 \leq 1/100$ and to prove that $\eps^{\rand}_n(\Learn \vert D,\eta) \geq \eta d / 32$.

%\shay{I can't parse the sentence beginning with ``Generally''. Also, I don't think ``generally'' is the right word here.}
\textbf{The VC dimensions is $d\geq 2$.}
Let $V=\{v_1, \dots, v_d\} \subset \cX$ be shattered by~$\cH$. Define a distribution~$D_{\cX}$ over $V$ as follows.
Set $D_{\cX}(v_i) = \eta/2$ for all $2 \leq i \leq d$, and set $D_{\cX}(v_1) = 1 - \eta (d-1) / 2$. 
Notice that $D_{\cX}$ is well defined since $d\geq 2$ and $\eta \leq 2/(d-1)$ (the latter is implied by the assumption that $\eta d/32 \leq 1/100$).
For any labeling function $\ell\in \cY^{V}$, let $D_{\ell}$ denote the distribution over examples defined 
by $D_{\ell}(v_i,\ell(v_i)) = D_{\cX}(v_i)$ for all $i\in [d]$. Note that $D_{\ell}$ is realizable, since $V$ is shattered.
It suffices to show that if the label vector $\ell \sim \cY^V$ is drawn uniformly at random then
\begin{equation} \label{eq:neg_res}
\E_{\ell \sim \cY^{V}}\E_{S \sim D_{\ell}^n,(x,y)\sim D_{\ell}}\!\left[ \sup_{S' \in B_{\eta}(S)} |\Learn(S')(x) - y| \right] \geq \eta (d-1) / 16.
\end{equation}
Indeed, the above implies that there exists $\ell \in \cY^{V}$ such that
\begin{align*}
\E_{S \sim D_{\ell}^n,(x,y)\sim D_{\ell}}\!\left[ \sup_{S' \in B_{\eta}(S)} |\Learn(S')(x) - y| \right] &\geq \eta (d-1) / 16\\
&\geq \eta d /32. \tag{$d\geq 2$}    
\end{align*}
% where we used the assumption that $d\geq 2$ in the last ineqaulity.

We establish Equation~\ref{eq:neg_res} in two steps:
\begin{enumerate}
    \item \label{itm:neg_easy} For a sample $S$ let $S^u$ be the unlabeled input sample underlying it.  We say that an unlabeled sample $S^u$ and an instance $x$ are \emph{hard} if $x\neq v_1$ and $x$ appears at most $\eta n$ times in $S^u$. In the first step we show that $\Pr_{\ell,S,(x,y)}[S^u,x \text{ are hard}] \geq \eta (d-1) / 4$.
    \item \label{itm:neg_hard} Let $E_2$ denote the event of all label vectors $\ell$, input samples $S$, and test examples~$(x,y)$ such 
     that $\sup_{S' \in B_{\eta}(S)} |\Learn(S')(x) - y| \geq 1/2$. In the second step we show that $\Pr[E_2| S^u,x \text{ are hard}] \geq 1/2$. 
\end{enumerate}
Indeed, once we prove both steps we have:
\begin{align*}
\E_{l \sim \cY^{V}}  \E_{S \sim D_{\ell}^{n}, (x,y) \sim D_{\ell}}\left[\sup_{S' \in B_{\eta}(S)} |\Learn(S')(x) - y|\right]
& \geq
\frac{1}{2} \cdot \Pr[E_2] \\
& \geq
\frac{1}{2} \cdot \Pr[S^u,x \text{ are hard}]  \cdot \Pr[E_2 | S^u,x \text{ are hard}] \\
& \geq
\frac{1}{2}\cdot \frac{\eta (d-1)}{4} \cdot \frac{1}{2}
 =
 \eta (d-1)/ 16,
\end{align*}
as desired.

Let us prove step~\ref{itm:neg_easy}. Notice that $S^u$ and $x$ are distributed according to the marginal distribution $D_{\cX}^{n+1}$. Thus, $x \neq v_{1}$ with probability $\eta (d-1) /2$, and given that $x \neq v_{1}$ the expected number of appearances of $x$ in $S^u$ is $\eta n /2$. 
Therefore, by Markov's inequality, the probability that $S^u$ and $x$ are hard given that $x \neq v_{1}$ is at least $\frac{\eta n / 2}{\eta n} = 1/2$. Thus, the overall probability that $S^u,x$ are hard is at least $\eta (d-1) / 4$.

%We will now prove Item~\ref{itm:neg_hard}. Given that $\hat{S}$ is hard and $x \neq v_{1}$, consider the samples $S_0,S_1 \in B_{\eta}(S)$, defined to be the same as $S$, with the exception that every example $S_i\in S$ for which $S_i = (x,y)$ is replaced in $S_0$ with the example $(v_{1}, 0)$. $S_1$ is defined similarly. 

We now prove step~\ref{itm:neg_hard}. Let $S^u, x$ be hard.
It suffices to show that
\[\E_{\ell(x_1),\ldots,\ell(x_n), y}\Bigl[\sup_{S' \in B_{\eta}(S)} \lvert\Learn(S')(x) - y\rvert ~ \Big\vert ~ S^u,x\Bigr] \geq \frac{1}{2},\]
where $\ell(x_i)$ is the label of the $i$'th instance in $S^u$ and $y$ is the test label.
Crucially, notice that the test-label $y$ is independent of $S^u$, $x$, and all other labels $\ell(x_i)$ for $x_i\in S^u$ such that $x_i\neq x$.
Thus, even conditioned on $S^u, x$ and all labels of $x_i\neq x$, the test-label $y$ is distributed uniformly in $\cY=\{0,1\}$. 

Define samples $S'_0, S'_1$ to be the same as $S'$ with the exception that every appearance of $x$ in $S'_0$ is labeled with $0$ in $S'_0$ and with $1$ in $S'_1$. Note that both $S'_0,S'_1\in B_{\eta}(S)$, because $S^u,x$ are hard.
We claim that, with probability at least half over the drawing of the $\ell(x_i)$'s and $y$ we have 
\[\lvert \Learn(S'_0)(x) - \ell(y) \rvert \geq 1/2  \quad \text{ or } \quad \lvert \Learn(S'_1)(x) - \ell(y) \rvert \geq 1/2.\]
Having this in hand, and given that $\hat{S}$ is hard, we are done: both $S'_0,S'_1 \in B_{\eta}(S)$, and Item~\ref{itm:neg_hard} follows. 

It thus remains to show that indeed $\lvert \Learn(S'_0)(x) - y \rvert \geq 1/2$ or $\lvert \Learn(S'_1)(x) - \ell(y) \rvert \geq 1/2$ with probability at least $1/2$ over the drawing of the $\ell(x_i)$'s and $y$. This is achieved by a simple case analysis:
\begin{itemize}
    \item if both $\Learn(S'_0)(x), \Learn(S'_1)(x) \leq 1/2$ then with probability $1/2$ we have $y=1$ and the claim follows. The case  $\Learn(S'_0)(x), \Learn(S'_1)(x) > 1/2$ is treated similarly.
    \item If $\Learn(S'_0)(x) \leq 1/2, \Learn(S'_1)(x) \geq 1/2$ then $\lvert \Learn(S'_0)(x) - y \rvert \geq 1/2$ or $\lvert \Learn(S'_1)(x) - y \rvert \geq 1/2$ with probability $1$ and the claim follows. The case $\Learn(S'_0)(x) > 1/2, \Learn(S'_1)(x) < 1/2$ is treated similarly.
\end{itemize}
This finishes the proof of Theorem~\ref{theo:neg} when the VC-dimension $d$ is at least $2$.

%the Note that $x$ does not appear in $S'$, and since $V$ is shattered both $l(x)=0$ and $l(x)=1$ are allowed whenever the input sample is $S'$. Thus, $l(x)$ is independent of $S'$, and therefore we can change the order of lotteries and draw $l(x) \sim \{0,1\}$ at the end. Now, since $\Learn(S')$ satisfies $\lvert \Learn(S')(x) \rvert \geq 1/2$ or $\lvert \Learn(S')(x) -1 \rvert \geq 1/2$, it holds that $\lvert \Learn(S')(x) - l(x) \rvert \geq 1/2$ with probability at least half, and Item~\ref{itm:neg_hard} follows.

\textbf{The VC-dimension is $d=1$.}
In this case, we can not define the distribution $D_{\cX}$ as before because $d<2$. However, the fact that $\cH$ is non-trivial allows to 
modify the definition as follows. Let $x_1,x_2\in \cX$ and $h_1,h_2\in \cH$ so that $h_1(x_1)=h_2(x_1)$ and $h_1(x_2) \neq h_2(x_2)$, guaranteed by the fact that $\cH$ is non-trivial. Set $V=\{x_1,x_2\}$, and define the distribution $D_{\cX}$ by $D_{\cX}(x_1) = 1- \eta/2, D_{\cX}(x_2) = \eta/2$ as in the case $d\geq 2$.  Also, define the random labeling function $\ell$ to agree with $h_1$ on with probability half and with $h_2$ with probability half. The rest of the proof is the same.
%
%
%Now we can prove the same items~\ref{itm:neg_easy}-\ref{itm:neg_hard} as in the case $d\geq 2$, only that in item~\ref{itm:neg_easy} we also require that $x_{n+1} = x_2$. This does not change the proof of item~\ref{itm:neg_easy}. The rest of the proof is the same. 
\end{proof}

\section{Proof of Theorem~\ref{theo:proper} (Realizable and Proper Case -- Positive Result)} \label{apndx:proper}

\begin{theorem*} [Restatement of Theorem~\ref{theo:proper}]
There exists a constant $c > 0$ so that the following holds. Let $\cH$ be the class of halfspaces over $\mathbb{R}^d$ for some $d \geq 1$, and let $\eta \in (0,1)$. Then, there exists a proper learner $\Learn$ having $\eta$-adversarial risk
\[
\eps^{\Adv}_{n}(\Learn | D, \eta) \leq c \eta d^3
\]
for any distribution $D$ realizable by $\cH$ and for any sample size $n \geq 1/\eta$.
\end{theorem*}

To derive Theorem~\ref{theo:proper}, we reinforce the $\SPV$ algorithm with a technique
introduced by~\citet*{kane2019communication} and further developed by \citet*{bousquet2020proper}.
This technique allows in certain cases to \emph{project} a majority vote of hypotheses from the class  $\cH$ back to $\cH$. 
Its applicability hinges on a combinatorial parameter called the \emph{projection number}:

\begin{definition}[Projection Number]
Let $\cH$ be a concept class. For any $\ell\geq 2$ and for any multiset $\cH' \subset \cH$ define the set $\cX_{\cH',\ell}$ to be the set of all $x\in \cX$, for which the number of hypotheses in $\cH'$ that disagree with $\Maj(\cH')(x)$ is less than $\lvert \cH' \rvert/\ell$. The Projection Number of the class $\cH$, denoted $k_p=k_p(\cH)$, is defined to be the smallest $\ell$ so that for any finite multiset $\cH' \subset \cH$, there exist $h\in \cH$ such that $h(x) = \Maj(\cH')(x)$ for all $x\in \cX_{\cH',\ell}$.  If no such $\ell$ exists then $k_p = \infty$.
\end{definition}

%\textcolor{red}{Shay: the definition of the projection number is a bit cumbersome; 
%since it is essentially equivalent to the \emph{hollow star number},
%I think we should state the proposition in terms of the latter.}

%
%The projection number of the class of halfspaces was discussed, in different formulations, also in \citep{kane2019communication, braverman2019convex}.

First, let us analyze the general performance of $\PSPV$.
%Before we prove Theorem~\ref{theo:proper}, we remark that a more general result can actually be proved.

%\begin{remark}
%One may use our technique to prove a version of Theorem~\ref{theo:proper} for any concpet class $\cH$ with VC-dimension $d$ and projection number $k_p$. The smallest possible value of $\eps$ achievable by this technique, is some increasing function of $\eta, d, k_p$, depending on the loss guarantee of a best proepr learner for $\cH$. \textcolor{red}{Shay: this remark is not clear.}
%\end{remark}

\begin{lemma}[General performance of $\PSPV$] \label{lem:PSPV_perf}
Let $\cH$ be a concept class with a finite projection number $k_p<\infty$. Let $D$ be a distribution over examples, and let $\Learn_p$ be a proper learning rule. Let also $\eta \in (0,1)$ be the stability parameter given to $\PSPV$ and let $n\geq 1/\eta$ be the sample size. Then $\PSPV(\Learn_p)$ is a proper learning rule having $\eta$-adversarial risk

\[
\eps^{\Adv}_n(\PSPV(\Learn_p) | D, \eta) \leq 4 k_p\eps_{ \lceil 1/(5 k_p \eta)  \rceil}(\Learn_p | D).
\]

\end{lemma}

\begin{proof}

The proof follows the same lines as the proof of Lemma~\ref{lem:SPV_perf}.
%\shay{The proof of this Lemma was modified, better modify this proof accordingly.}
 Let $S \sim D^n$ be the input sample, and $(x,y) \sim D$ be the test example. Note that for all $i\in [t]$ (where $t= \lfloor 5 k_p \eta n \rfloor$ is the number of subsamples of size at least $\frac{1}{5 k_p \eta}$ in the partition made by $\PSPV$) it holds that $\E\bigl[1[h_i(x) \neq y]\bigr] \leq \eps_{\lceil 1/(5 k_p \eta) \rceil}(\Learn_p | D)$. By applying linearity of expectation we get

\[
\E\left[\frac{1}{t} \sum_{i=1}^t 1[h_i(x) \neq y] \right] \leq \eps_{\lceil 1/(5 k_p \eta) \rceil}(\Learn_p | D).
\]

By Markov's inequality:

\[
\Pr \left[\frac{1}{t} \sum_{i=1}^t 1[h_i(x) \neq y] \geq \frac{1}{4 k_p } \right] \leq 4 k_p  \eps_{\lceil 1/(5 k_p \eta) \rceil}(\Learn_p | D).
\]

Let $S' \in B_{\eta}(S)$. Let $h' = \PSPV(\Learn_p)(S')$, and for all $i\in [t]$ let $h'_i$ be the hypothesis obtained by training $\Learn_p$ on $S'^{(i)}$. Note that, since $S$ and $S'$ are $\eta$-close by, and since $n \geq 1/\eta$ it holds that
\[
\frac{1}{t}\sum _{i=1}^t 1\left[S^{(i)} \neq S'^{(i)}\right] \leq
\frac{\eta n}{\lfloor 5 k_p \eta n \rfloor}
\leq
\frac{1}{4 k_p }.
\]

Hence it is implied that $\frac{1}{t}\sum _{i=1}^t 1\left[h_i(x) \neq h'_i(x)\right] \leq \frac{1}{4 k_p }$. Thus, the event that $\frac{1}{t} \sum_{i=1}^t 1[h'_i(x) \neq y] \geq \frac{1}{2 k_p }$ implies (or, is contained in) the event that  $\sum_{i=1}^t 1[h_i(x) \neq y] \geq \frac{1}{4 k_p }$, hence:

\[
\Pr \left[ \frac{1}{t}\sum_{i=1}^t 1[h'_i(x) \neq y] \geq \frac{1}{2 k_p } \right] 
\leq
4 k_p  \eps_{\lceil 1/(5 k_p \eta) \rceil}(\Learn_p | D).
\]
Note that by definition of projection number it holds that the hypothesis $h'\in \cH$ returned by the algorithm exists.
%Moreover, since $k_p \geq 2$ it holds that $2k_p > 2$.
Hence, by definition of $\cX_{\{h'_1, \dots, h'_t\},2k_p}$ the above implies that
\[
\Pr[h'(x) \neq y] \leq 4 k_p  \eps_{\lceil 1/(5 k_p \eta) \rceil}(\Learn_p | D).
\]

Since $S'$ is an arbitrary sample in $B_{\eta}(S)$, the above implies that $\PSPV(\Learn_p)$ has the stated $\eta$-adversarial risk.
\end{proof}

To prove Theorem~\ref{theo:proper} we will use the following result regarding the projection number of halfspaces.

\begin{theorem}[\citet*{kane2019communication,braverman2019convex,bousquet2020proper}] \label{theo:lc_projection}
Let $\cH$ be the class of halfspaces over $\mathbb{R}^m$. Then $k_p(\cH) = d(\cH) = m+1$.
\end{theorem}

We will use the SVM learner as an input learner for $\PSPV$.
\begin{theorem}[\citet*{vapnik1974theory}] \label{theo:svm}
Let $m \geq 1$ and let $\cH$ be the class of halfspaces over $\mathbb{R}^m$. Let $D$ be a distribution realizable by $\cH$. Let also $n\in \mathbb{N}$, and let $\Learn_p$ be the SVM algorithm. Then $\eps_n(\Learn_p | D) \leq \frac{m+1}{n+1}$.
%for any distribution $D$ over examples which is realizable by $\cH$, if $S \sim D^n$ then $L_D(\Learn(S)) \leq C\frac{d}{n}$.
%\textcolor{red}{Shay: The bound on the population loss here is $\frac{d+1}{n+1}$ and I think appeared already in Vapnik and Chervonenkis's book from the 70's.}
\end{theorem}

Theorem~\ref{theo:proper} now follows as an immediate application of Theorem~\ref{theo:lc_projection}, Theorem~\ref{theo:svm} and Lemma~\ref{lem:PSPV_perf}.

\begin{corollary}[Realizable and proper case -- positive result]
Let $m\geq 1$, let $\cH$ be the class of halfspaces over $\mathbb{R}^m$, and let $d=m+1$ be the VC-dimension of $\cH$. Let $D$ be a distribution realizable by $\cH$, and let $\Learn_p$ be the SVM learner. Let also $\eta \in (0,1)$ be the stability parameter given to $\PSPV$ and let $n \geq 1/\eta$ be the sample size. Then $\PSPV(\Learn_p)$ has $\eta$-adversarial risk
\[
\eps^{\Adv}_n(\PSPV(\Learn_p) | D, \eta) \leq 20 \eta d^3.
\]

\end{corollary}

\begin{proof}
By Theorem~\ref{theo:lc_projection}, if $\cH$ is the class of halfspaces over $\mathbb{R}^m$ then its projection number is $k_p=d=m+1$. Also, by Theorem~\ref{theo:svm}, we have that $\eps_{\lceil 1/(5 d \eta) \rceil}(\Learn_p | D) \leq 5 \eta d^2$. Plug both results to Lemma~\ref{lem:PSPV_perf}, and the result follows.
\end{proof}

\section{Proof of Theorem~\ref{theo:pos_semi_agn} (Agnostic Case -- Positive Result)} \label{apndx:agn_pos}

\begin{theorem*}[Restatement of Theorem~\ref{theo:pos_semi_agn}]
There exist constants $c_1,c_2$ so that the following holds. Let $\cH$ be a hypothesis class with VC dimension $d$ and let $\eta\in (0,1)$. 
Then, there exists a learner $\Learn$ having $\eta$-adversarial risk 
\[\eps^{\Adv}_n(\Learn | D, \eta) \leq c_2\cdot\OPT + c_1\cdot d\cdot \eta\] 
for any distribution $D$ over examples and for any sample size $n \geq 1/\eta$.
\end{theorem*}

To derive Theorem~\ref{theo:pos_semi_agn}, we use an agnostic variation of the One-inclusion graph learner.

\begin{theorem}[Corollary of Lemma~16 in \citep*{long1999complexity}] \label{theo:semi_one_inc}
There exists a constant $C$ such that the following holds. Let $\cH$ be a concept class with VC-dimension $d$ and let $\Learn$ be the agnostic variation of the One-inclusion graph algorithm implied by Lemma~16 in \citep*{long1999complexity}. Let also $n$ be the sample size. Then, for any distribution $D$ over examples (not necessarily such that is realizable by $\cH$), it holds that $\eps_{n}(\Learn | D) \leq C(\OPT + d/n)$.
\end{theorem}

%\begin{remark}
%By Item~\ref{obs:eqv_rob_stab_itm1} in Observation~\ref{obs:eqv_rob_stab}, it is implied\footnote{This observation is stated in terms of the realizable case definitions, but applicable for the agnostic case as well, by multiplying the semi-agnostic constant by a factor of $2$.} that if $\frac{\min\{\eps, \sigma \}}{\eta} \geq c\cdot d$, then $\cH$ is semi-agnostic $(\eps + \sigma)$-learnable under $\eta$-targeted poisoning attacks. Semi-agnostic $\eps'$-learnability under $\eta$-targeted poisoning attacks means that there exists a learner $\Learn$, a sample size $n_0$ and a constant $C$ such that for any adversary changing up to an $\eta$-fraction of the input sample, for any distribution $D$ and for any $n \geq n_0$ it holds that $\Pr_{S \sim D^n,(x,y) \sim D}[\Learn(S')(x) \neq y] \leq C\cdot \OPT + \eps'$, where $S'$ is the sample $S$ after being tampered by the adversary.
%\end{remark}

Theorem~\ref{theo:pos_semi_agn} is implied by the following immediate corollary of Theorem~\ref{theo:pos_semi_agn} and Lemma~\ref{lem:SPV_perf}.

\begin{corollary} [Agnostic case -- positive result] \label{cor:semiAg}
There exists a constant $C$ such that the following holds. Let $\cH$ be a concept class with $VC$ dimension $d$, let $\eta \in (0,1)$ be the stability parameter given to $\SPV$, and let $D$ be a (not necessarily realizable) distribution over examples. Let also $n\geq 1/\eta$ be the sample size. Then $\SPV(\Learn)$ has $\eta$-adversarial risk
\[
\eps^{\Adv}_n(\SPV(\Learn) | D, \eta) \leq 6C\OPT + 42 C \eta d,
\]
where $\Learn$ is the agnostic variant of the One-inclusion graph algorithm mentioned in Theorem~\ref{theo:semi_one_inc}.
\end{corollary}

\begin{proof}
By Theorem~\ref{theo:semi_one_inc}, there exists a constant $C$ such that $\eps_{\lceil 1/(7\eta) \rceil}(\Learn | D) \leq C\OPT + 7 C \eta d$. Plug this into Lemma~\ref{lem:SPV_perf} and the result follows.
\end{proof}

\section{Proof of Theorem~\ref{theo:agn_neg} (Agnostic Case -- Impossibility Result)}\label{apndx:agn_neg}

\begin{theorem*}[Restatement of Theorem~\ref{theo:agn_neg}]
Let $\eta' \in (0,1), n\in \mathbb{N}$. For any hypothesis class $\cH$ that has at least two hypotheses   and. for any deterministic learner, there is a distribution $D$ over (two) examples and $\eta=\eta'+\widetilde{O}(1/\sqrt{n})$ such that $\Learn$ has $\eta$-adversarial risk
\[
\eps^{\Adv}_n(\Learn |D , \eta) \geq 2\OPT+\Omega(\eta')-O(1/n).
\]
\end{theorem*}

Let $h_1,h_2\in \cH$ be two distinct hypotheses and let $x\in \cX$ such that $h_1(x)\neq h_2(x)$.
In this proof we consider distributions $D$ supported only on $\{(x,0),(x,1)\}$.
Notice that such a distribution is determined by the probability $p = \Pr_{(x,y)\sim D}[y=1]$ and hence can be thought of as a coin with bias $p$. Thus, the task of agnostic learning such distributions with respect to instance-targeted data poisoning boils
down to predicting a random $p$-coin toss given an input sample of $n$ $p$-coin tosses
out of which at most $\eta\cdot n$ tosses are flipped by an adversary who \emph{knows}
the result of the coin toss that needs to be predicted. We summarize this in the following game:
%
%In order to prove Theorem~\ref{theo:agn_neg}, we only need an extremely simple setting of VC dimension at least $1$ (i.e., $|\cH | \geq 1$). Namely, we use any point $x$ where $h_0(x) \neq h_1(x)$ for $h_0,h_1 \in \cH$. We will work with a distribution $D$ withi support set  $\set{(x,0),(x,1)}$, uniquely defined by $p=\Pr[D_p=(x,1)]$. In other words,  $D_p$ merely defines a \emph{coin} of an \emph{unknown} bias $p$, and the job of the learner is to   find out about $p$ with the goal of predicting a fresh test coin. 
%
%Below, a $p$-coin is modeled by a boolean random  variable $X_p$ where $\Pr[X=1]=p$.
%We  now  define a two party game between an adversary and a learner. The adversary   picks $p$ and the learner   tries minimize the ``regret'' under the instance targeted poisoning, when its performance is compared with the best function in class $\cH$. We simplify the game to only use a parameter $p$ and use a  $p$-coin random variable $X_p$, while more formally  $p$ determines the distribution $D_p$ over $(x,0),(x,1)$.

\begin{definition}[The coin   game] \label{def:coin} The coin game is parameterized by $(n, \eta)$ where $ n \in \mathbb{N}, \eta \in (0,1)$, and the game is played  between an adversary $\Adv$  and a learner $\Learn$ as follows.
\begin{enumerate}
    \item $\Adv$ picks $p \in [0,1]$.
    \item  $c_1,\dots,c_{n+1} \sim X_p^{n+1}$, where $X_p$ is a binary random variable satisfying $\Pr[X_p=1]=p$.
    \item \label{step:corruption} $\Adv$ changes  $\ol{c}=(c_1,\dots,c_n)$ into $\ol{c}'=(c'_1,\dots,c'_n)$ where $\dist(\ol{c},\ol{c}')\leq \eta \cdot n$.
    \item   $\Learner$ gets to see $\ol{c}'=(c'_1,\dots,c'_n)$ and   outputs a bit $c\in\{0,1\}$.
    \item $\Learner$  wins if $c = c_{n+1}$, and $\Adv$ wins otherwise.
\end{enumerate}
In this game, we define $\OPT_p = \min\set{p,1-p}$ to be the optimal error of the learner if it had known $p$, and we define $\ERR = \Pr[c \neq c_{n+1}]$ (over all the randomness involved) to be the \emph{error} of the game (i.e., when the learner does not win). We also refer to $\ERR-\OPT_p$ as the regret.\footnote{Note that $\OPT_p$ is a random variable in general, if the adversary is randomized. But if the adversary uses a deterministic strategy for the fixed $p$, then $\OPT_p$ is a constant.}
\end{definition}

%First we show that if the adversary aims to maximize the expected regret, for every \emph{fixed} learner $\Learner$, without loss of generality  the adversary can always employ a deterministic strategy; that is because one can fix the adversary's randomness to what is best for it. \Inote{Why without loss of generality?} \Mnote{I added the important point that the learner is fixed.}
\begin{theorem} \label{thm:noAgn}
For any $\eta' \in [0,1/2]$ and any deterministic learner $\Learn$ that participates in the coin game of Definition~\ref{def:coin},   there is an adversary $\Adv$ with a fixed choice of $p$ (determining $\OPT=\OPT_p$) and $\eta=\eta' + \widetilde{O}(1/\sqrt n)$ such that when we run the game of Definition~\ref{def:coin} with parameters $(n,\eta)$, it holds that $\ERR-\OPT \geq 1/2 + \eta' - O(1/n)$.
\end{theorem}

\begin{remark}[On deterministic adversaries]
In Theorem~\ref{thm:noAgn} we show the existence of an adversary with a fixed choice of $p$. This adversary is in fact randomized. Here we remark that, for every fixed (even randomized) learner $\Learner$ and a fixed  choice of $p$, there is always a deterministic adversary that achieves the maximum regret (for such $\Learner,p$). The reason is that if by using randomness $r_\Adv$ the adversary achieves expected regret $R(r_\Adv)$ over the randomness of the learner, then its overall regret will be $\Ex_{r_\Adv}[R(r_\Adv)]$. Therefore, if $r_\Adv^{(p)}$ is the randomness (for fixed $p$) that maximizes $R(r_\Adv)$, the adversary can simply fix its randomness to $r_\Adv^{(p)}$ without decreasing its gain. This means that without loss of generality, the adversary of   Theorem \ref{thm:noAgn} is deterministic. In addition, since the adversary sends the first message $p$, the overall optimal strategy $\Adv$ (who picks $p$ potentially in a randomized way) can \emph{also} fix $p$ to what maximizes $R(r^{(p)}_\Adv)$,  which makes $\Adv$ fully deterministic.
\end{remark}

\textbf{Deriving Theorem~\ref{theo:agn_neg}.} 
We first show how to derive Theorem~\ref{theo:agn_neg} from Theorem~\ref{thm:noAgn}.
\begin{proof}[Proof of Theorem~\ref{theo:agn_neg}]
First assume $\eta' \leq 1/2$, and at the end we explain how to deal with $\eta' > 1/2$.  By Theorem~\ref{thm:noAgn}, there is an adversary (with a fixed choice of $p$) in the coin game of Definition~\ref{def:coin} such that $\ERR - \OPT \geq 1/2 + \eta' - O(1/n)$ when we use $(n,\eta)$ as game parameters.     Since $\ERR \leq 1$, we have $\OPT \leq 1/2 - \eta' +O(1/n)$, and so 
$$\ERR - \OPT \geq 1/2 + \eta' - O(1/n) \geq  \OPT + 2\eta' -O(1/n).$$ 
This implies that $\ERR \geq 2\OPT + 2\eta' -O(1/n)$. Note that $\OPT$ is indeed the minimal error that the learner can achieve by outputting any of the constant coins $0,1$, which in turn refers to outputting either of $h_0,h_1$ from the hypothesis class. In addition, $\ERR$ is equal to the adversarial risk for parameters $n,\eta$ and the distribution $D_p$ for this particular attack. This means that
$$\eps^{\Adv}_n(\Learn |D , \eta) \geq 2\OPT+ 2\eta' -O(1/n),$$
which implies Theorem~\ref{theo:agn_neg}.
Now, if $\eta' > 1/2$, we first artificially decrease adversary's budget $\eta'$ to $\eta''=1/2$, which leads to 
$$\eps^{\Adv}_n(\Learn |D , \eta) \geq 2\OPT+ 2\eta'' -O(1/n),$$
but we also know that $\eta'' = \Omega(\eta')$, which again proves Theorem~\ref{theo:agn_neg}.
\end{proof}

Before proving Theorem~\ref{thm:noAgn} we recall two useful tools.

\begin{lemma} [Proposition 2.1.1 in~\citet*{talagrand1995concentration}] \label{lem:Tal}
Let $\mu = \mu_1\times \dots \mu_n$ be a product measure and $f \colon \mu \To \bits$  a  boolean function where $\Pr[f(\mu)=1]=1/2$. Then, for all $b \in [n]$,
$$\Pr_{x \sim \mu}[\exists x', \dist(x,x')\leq b \land f(x')=1] \geq 1- 2 e^{-b^2/n}.$$ In other words, with probability at least $1- 2 e^{-b^2/n}$ over the sampling of $x \sim \mu$, one can change up to $b$ of the coordinates of $x$ and obtain $x'$ (i.e., $\dist(x,x') \leq b$) such that $f(x')=1$.
\end{lemma}
\begin{lemma}[Modifying coins] \label{lem:shift} Suppose  $0\leq p,p' \leq 1$, and let $q=|p-p'|$. Then there is an adversary who can change $q\cdot n$ coins, in expectation, of a sample $\ol{c} \sim X^n_p$ into $\ol{c}'$ (i.e., $\E[\dist(\ol{c}',\ol{c})] = q \cdot n$) such that $\ol{c}'\sim X^n_{p'}$ (Namely, the tampered sequence looks exactly like it  is  sampled from $X^n_{p'}$, while in reality it is being first sampled from $X^n_p$ and then modified by the adversary in $q\cdot n$ points in expectation). Moreover, the probability that the adversary changes more than $q n+ \sqrt {(n \ln n)/{2}}$ of the coordinates is at most $1/n$.
\end{lemma}
\begin{proof}
Without loss of generality, let $p'-p = q\geq 0$. Then the adversary will  change each of the coins with independent probability $q$  as follows. If a coin $c_i=1$, the adversary will not change it, which will happen with probability $p$. If $c_i=0$, which will happen with probability $1-p$, the adversary will change this to $1$ with probability $q/(1-p)$ over its own randomness. Note that $ q = p'-p \leq (1-p)$, and so $q/(1-p) \in [0,1]$ can be interpreted as a probability. The   probability that   $c'_i=1$ is now exactly $p+q=p'$, while the expected number of changed coins is $q \cdot n$. Finally, since the adversary's changes of the coin outcomes are done \emph{independently} for each coin, the bound on the number of changes made by the adversary is implied by   the Hoeffding-Chernoff  bound.
\end{proof}

We now prove Theorem~\ref{thm:noAgn} using the two tools above.
\begin{proof}[Proof of Theorem~\ref{thm:noAgn}]
Fix the deterministic learning algorithm $\Learn$. This means that for every given input vector $\ol{c} =(c_1,\dots,c_n)$, we have $\Learn(\ol{c}) \in \bits$. Now define $\alpha(p)=\Pr_{\ol{c} \sim X^{n}_p}[\Learn(\ol{c}) = 1]$. 

We do a case study as follows.

\begin{itemize}
    \item If $\alpha(0)\neq 0$, it  means that $\alpha(0)=1$ (i.e., the deterministic learner outputs $1$ over the all zero vector). In this case, $\OPT=0$ and $\ERR=1$, which implies $\ERR-\OPT \geq 1/2 + \eta'$.
    \item If $\alpha(1)\neq 1$, it  implies $\ERR-\OPT \geq 1/2 + \eta'$ similarly.
    \item If none of the above cases happens, we can assume $\alpha(b)=b$ for both $b \in \bits$.   Because the learner is deterministic, $\Learn(\ol{c})=1$  if  $\ol{c} \in \cS$ for a fixed set $\cS \subseteq \bits^{n}$. Moreover, for all $\ol{c} \in \bits^n$, it holds that $\Pr[X_p^n = \ol{c}] = p^{d}(1-p)^{n-d}$, where $d$ is the number of non-zero coordinates of $\ol{c}$. This implies that $\alpha(p)$ is a polynomial of degree at most $n$ over $p$, which is a continuous function. Therefore, there exists $q \in (0,1)$ such that $\alpha(q)=1/2$. Without loss of generality, assume that $q \leq 1/2$. Then, the adversary   picks $p = \max \set{0,q-\eta'}$, which guarantees $\OPT=p \leq 1/2-\eta'$ (due to the assumptions $\eta', q \leq 1/2$). Then, the adversary uses Lemma~\ref{lem:shift} to shift the coin's distribution back to $q$. For this change, the adversary makes at most $\eta' \cdot n + \sqrt {(n \ln n)/{2}}$ changes with probability $1-1/n$. 
    %One crucial point is that even though the coins are being tampered with using the attacker of Lemma~\ref{lem:shift}, they are still \emph{independent} because of how the attacker of Lemma~\ref{lem:shift} 
    We then apply the algorithm of Lemma \ref{lem:Tal} to make further $\sqrt{n \ln (2n)}$ changes to the coins to make sure that the output of the learner is the wrong outcome (different from $c_{n+1}$) with probability $1-1/n$. %Now, if we simply cut adversary of Lemma \ref{lem:shift} to at most $\sqrt {(n \ln n)/{2}}$  changes (while still applying the adversary of Lemma \ref{lem:Tal}), 
    In total, the adversary can make at most $\eta' \cdot n + \sqrt {(n \ln n)/{2}} + \sqrt{n \ln (2n)} \in \eta' \cdot n + \widetilde{O}(\sqrt{n})$ changes to the coin flips outcomes, while the learner's output bit is wrong with probability $1-1/n-1/n = 1-O(1/n)$.
    Since $\OPT \leq 1/2 - \eta'$ and $\ERR \geq 1-O(1/n)$, we get 
    $$\ERR - \OPT \geq 1/2 + \eta' - O(1/n),$$
    which finishes the proof. \qedhere
\end{itemize}
\end{proof}

\end{document}